\documentclass[sigconf]{acmart}
\AtBeginDocument{%
  }

\setcopyright{none}
\usepackage{microtype}
\usepackage{graphicx}
\usepackage{subcaption}
\usepackage{booktabs} % for professional tables
\usepackage[table]{xcolor}
\usepackage{multirow}
\usepackage{tabularx}
\usepackage{float} % for precise table placement

\usepackage{hyperref}

\usepackage{tabularx}
\usepackage{xltabular}  % for long tables
\definecolor{tableheader}{HTML}{EDF1F6}
\definecolor{rowhighlight}{HTML}{E8F2FF}
\definecolor{bestblue}{HTML}{1F4E79}
\definecolor{secondorange}{HTML}{B35C00}
\definecolor{DomRec}{HTML}{FCE4D6}    % 柔和橙色 - RecSys
\definecolor{DomTS}{HTML}{DEEBF7}     % 柔和蓝色 - Time Series  
\definecolor{DomGML}{HTML}{E2EFDA}    % 柔和绿色 - Graph Learning
\newcommand{\best}[1]{\textcolor{bestblue}{\textbf{#1}}}

\newcommand{\method}{\textsc{PaperRepro}}

\usepackage{amsmath}

\usepackage{amssymb}
\usepackage{pifont}  % for \ding checkmarks
\usepackage{mathtools}
\usepackage{amsthm}
\usepackage{algorithm}
\usepackage{algorithmic}
\usepackage{enumitem}

\DeclareMathOperator*{\argmin}{arg\,min}
\newcommand{\Ncand}{\mathcal{N}_{\text{cand}}}
\newcommand{\Nprune}{\mathcal{N}_{\text{prune}}}

\usepackage[capitalize,noabbrev]{cleveref}

\usepackage{tikz}

\newcommand{\challenge}[1]{\tikz[baseline=(char.base)]{\node[shape=circle,fill=orange!30,inner sep=2pt] (char) {\footnotesize\bfseries #1};}}

\usepackage{listings}
\usepackage{xcolor}
\usepackage[most]{tcolorbox} % most 已经包含 skins/breakable 等常用库

\definecolor{codebg}{rgb}{0.985,0.985,0.985}      % very light gray (still clean)
\definecolor{codeframe}{rgb}{0.35,0.35,0.35}      % darker frame
\definecolor{codetitlebg}{rgb}{0.25,0.25,0.25}    % dark title bar
\definecolor{codetitlefg}{rgb}{1.00,1.00,1.00}    % white title text
\definecolor{patchbg}{rgb}{0.98,0.86,0.55}        % deeper yellow for 
\definecolor{patchred}{rgb}{0.8,0,0}        % 红色 for deletion/removed
\definecolor{patchgreen}{rgb}{0,0.8,0}        % 绿色 for addition/added

\lstdefinestyle{paperpy}{
  language=Python,
  basicstyle=\ttfamily\fontsize{6.2}{7.2}\selectfont,
  numbers=left,
  numberstyle=\ttfamily\tiny\color{codeframe},
  numbersep=6pt,
  xleftmargin=1.6em,
  frame=none,
  breaklines=true,
  breakatwhitespace=true,
  columns=fullflexible,
  keepspaces=true,
  showstringspaces=false,
  tabsize=2,
  escapeinside={(*@}{@*)} % IMPORTANT: allow LaTeX inside listings for highlighting
}
\lstdefinelanguage{diff}{
  morecomment=[f][\color{patchred}]{-},
  morecomment=[f][\color{patchgreen}]{+},
  morecomment=[f][\color{codeframe}]{@},
}

\lstdefinestyle{paperdiff}{
  language=diff,
  basicstyle=\ttfamily\fontsize{6.2}{7.2}\selectfont,
  numbers=left,
  numberstyle=\ttfamily\tiny\color{codeframe},
  numbersep=6pt,
  xleftmargin=1.6em,
  frame=none,
  breaklines=true,
  breakatwhitespace=true,
  columns=fullflexible,
  keepspaces=true,
  showstringspaces=false,
  tabsize=2,
  escapeinside={(*@}{@*)} % IMPORTANT: allow LaTeX inside listings for highlighting
}

\newtcolorbox{codebox}[1]{
  enhanced,
  colback=codebg,
  colframe=codeframe,
  boxrule=0.6pt,
  arc=1.5pt,
  left=2pt,right=2pt,top=2pt,bottom=2pt,
  title=\textbf{#1},
  colbacktitle=codetitlebg,
  coltitle=codetitlefg,
  fonttitle=\small,
}

\newtcolorbox{codeboxHL}[2][]{%
  enhanced,
  colback=codebg,
  colframe=codeframe,
  boxrule=0.6pt,
  arc=1.5pt,
  left=2pt,right=2pt,top=2pt,bottom=2pt,
  title=\textbf{#2},
  colbacktitle=codetitlebg,
  coltitle=codetitlefg,
  fonttitle=\small,
  width=\linewidth,
  valign=top,
  #1
}

\lstdefinestyle{prompttext}{
  language={},
  basicstyle=\ttfamily\fontsize{6.4}{7.4}\selectfont,
  numbers=left,
  numberstyle=\ttfamily\tiny\color{black!55},
  frame=none,
  breaklines=true,
  keepspaces=true,
  showstringspaces=false,
  xleftmargin=2em,
  framexleftmargin=2em,
  numbersep=5pt,
  columns=fullflexible,
  escapechar=,
  escapeinside={},
}

\newtcolorbox{promptbox}[2][]{%
  enhanced,
  breakable,
  colback=black!2,
  colframe=black!55,
  boxrule=0.6pt,
  arc=1.5pt,
  left=3pt,right=3pt,top=3pt,bottom=3pt,
  colbacktitle=black!80,
  coltitle=white,
  title=\textbf{#2},
  fonttitle=\small,
  width=\linewidth,
  pad at break=2mm,
  before skip=6pt,
  after skip=10pt,
  #1
}

\theoremstyle{plain}
\newtheorem{theorem}{Theorem}[section]
\newtheorem{proposition}[theorem]{Proposition}

\theoremstyle{definition}

\newtheorem{assumption}[theorem]{Assumption}
\theoremstyle{remark}

\usepackage[textsize=tiny]{todonotes}

\begin{document}

\title{What Papers Don't Tell You: Recovering Tacit Knowledge for Automated Paper Reproduction}

\author{Lehui Li}
\authornote{Both authors contributed equally to this research.}
\affiliation{%
  \institution{School of Software, Shandong University}
  \country{China}
}

\author{Ruining Wang}
\authornotemark[1]
\affiliation{%
  \institution{School of Software, Shandong University}
  \country{China}
}

\author{Haochen Song}
\affiliation{%
  \institution{School of Software, Shandong University}
  \country{China}
}

\author{Yaoxin Mao}
\affiliation{%
  \institution{Beijing Institute of Technology}
  \country{China}
}

\author{Tong Zhang}
\affiliation{%
  \institution{Zhejiang University}
  \country{China}
}

\author{Yuyao Wang}
\affiliation{%
  \institution{Dept.\ of Math \& Statistics, Boston University}
  \country{USA}
}

\author{Jiayi Fan}
\affiliation{%
  \institution{School of Software, Shandong University}
  \country{China}
}

\author{Yitong Zhang}
\affiliation{%
  \institution{College of AI, Tsinghua University}
  \country{China}
}

\author{Jieping Ye}
\affiliation{%
  \institution{Alibaba Group}
  \country{China}
}

\author{Chengqi Zhang}
\affiliation{%
  \institution{The Hong Kong Polytechnic University}
  \country{Hong Kong SAR, China}
}

\author{Yongshun Gong}
\authornote{Corresponding author.}
\affiliation{%
  \institution{School of Software, Shandong University}
  \country{China}
}

\renewcommand{\shortauthors}{Li et al.}

\begin{abstract}
Automated paper reproduction---generating executable code from academic papers---is bottlenecked not by information retrieval but by the tacit knowledge that papers inevitably leave implicit. We formalize this challenge as the progressive recovery of three types of tacit knowledge---relational, somatic, and collective---and propose \method, a graph-based agent framework with a dedicated mechanism for each: node-level relation-aware aggregation recovers relational knowledge by analyzing implementation-unit-level reuse and adaptation relationships between the target paper and its citation neighbors; execution-feedback refinement recovers somatic knowledge through iterative debugging driven by runtime signals; and graph-level knowledge induction distills collective knowledge from clusters of papers sharing similar implementations. On an extended ReproduceBench spanning 3 domains, 10 tasks, and 40 recent papers, \method{} achieves an average performance gap of 10.04\% against official implementations, improving over the strongest baseline by 24.68\%. The code will be publicly released upon acceptance; the repository link will be provided in the final version.
\end{abstract}




\maketitle

\section{Introduction}
Automated paper reproduction aims to generate executable code implementations from academic papers, substantially reducing the cost of validating and reusing prior scientific findings and accelerating the advance of scientific knowledge~\citep{peng2011reproducible,stodden2016enhancing,xu2023artificial,seo2025paper2codeautomatingcodegeneration,zhao2025autoreproduceautomaticaiexperiment,starace2025paperbenchevaluatingaisability,jimenez2024swebenchlanguagemodelsresolve}.

Early approaches, such as workflow-driven agents, encode expert knowledge into fixed workflows~\citep{seo2025paper2codeautomatingcodegeneration,li2025deepcodeopenagenticcoding,chen2025deepreproducer,lin2025autop2cllmbasedagentframework}. While achieving a degree of automation, they neglect the relational structure among papers, leading to inferior performance~\citep{zhao2025autoreproduceautomaticaiexperiment,hu2025graggraphretrievalaugmentedgeneration}. More recently, graph-augmented retrieval agents retrieve relevant text or code snippets from the citation neighborhood of a paper and directly concatenate them as contextual input~\citep{lewis2021retrievalaugmentedgenerationknowledgeintensivenlp,zhao2025autoreproduceautomaticaiexperiment,edge2025localglobalgraphrag,hu2025graggraphretrievalaugmentedgeneration}. However, since scientific papers are primarily written for peer communication rather than reproduction, authors routinely omit numerous implementation details, giving rise to \emph{tacit knowledge}~\citep{Gundersen_Kjensmo_2018,Gundersen_Gil_Aha_2018,pineau2020improvingreproducibilitymachinelearning,zhao2025autoreproduceautomaticaiexperiment}. This compels us to ask:

\begin{center}
\textit{What papers don't tell you?}
\end{center}

Understanding the tacit knowledge that papers leave unexpressed can provide valuable guidelines and insights for the development of advanced automated paper reproduction~\citep{zhao2025autoreproduceautomaticaiexperiment,shull2002replicating}. According to Collins' taxonomy, tacit knowledge involved in scientific reproducibility falls into three categories, each progressively more difficult to acquire~\citep{collins2019tacit}:

\begin{figure}[t]
\centering
\includegraphics[width=\columnwidth]{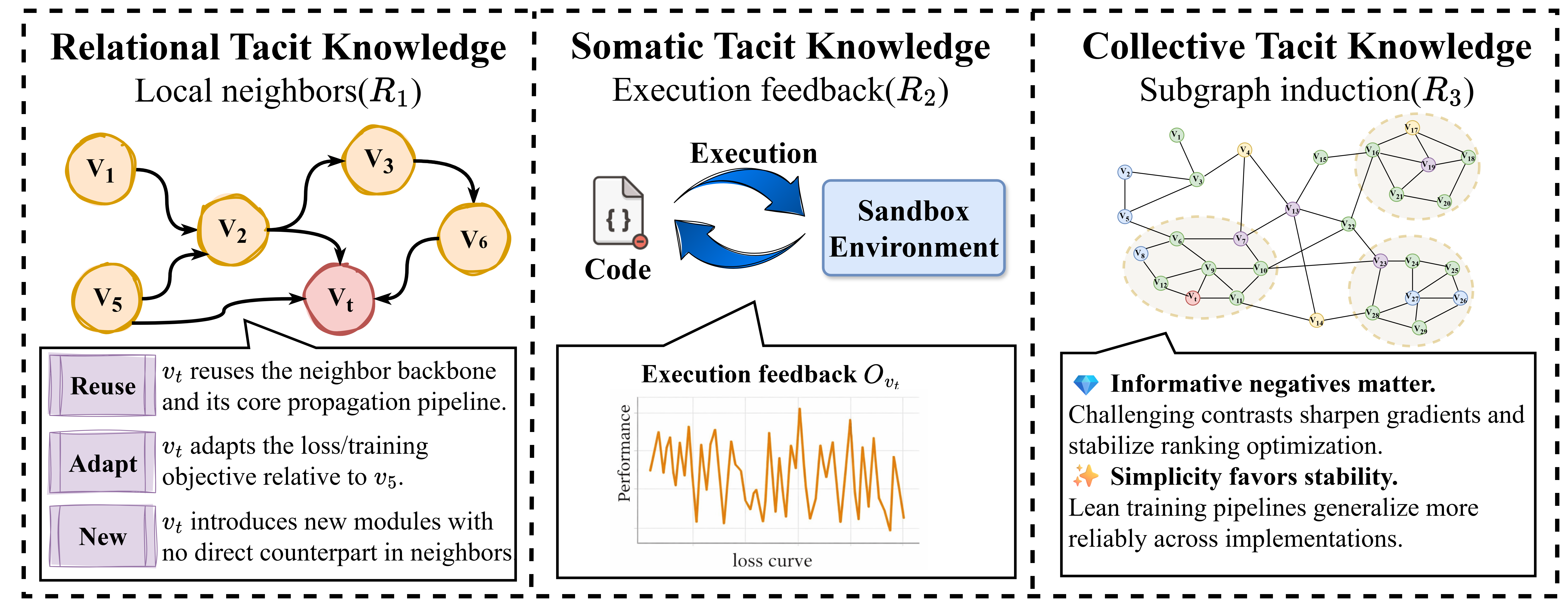}
\caption{Illustrative examples of the three types of tacit knowledge in automated paper reproduction. \textbf{Left}: Relational---implementation units are reused, adapted, or newly created relative to neighbor papers. \textbf{Middle}: Somatic---execution feedback from a sandbox environment guides iterative refinement. \textbf{Right}: Collective---common practices are induced from clusters of related implementations.}
\label{fig:motivation}
\end{figure}

\challenge{1}~\emph{Relational Tacit Knowledge} (Figure~\ref{fig:motivation}, left): authors often omit implementation details when they assume readers share familiarity with prior work. This knowledge concerns how prior implementations are reused, adapted, or extended, and can usually be recovered by analyzing the relationship between the target paper and its related work~\citep{collins2019tacit,zhao2025autoreproduceautomaticaiexperiment}.

\challenge{2}~\emph{Somatic Tacit Knowledge} (Figure~\ref{fig:motivation}, middle): knowledge rooted in a practitioner's experience and intuition, difficult to articulate explicitly, but recoverable through iterative experimentation and refinement based on observed runtime behavior~\citep{collins2019tacit,jimenez2024swebenchlanguagemodelsresolve,yang2024sweagentagentcomputerinterfacesenable,wang2025openhandsopenplatformai}.

\challenge{3}~\emph{Collective Tacit Knowledge} (Figure~\ref{fig:motivation}, right): knowledge embedded in the collective practices of a research community, difficult for any individual to articulate explicitly~\citep{collins2019tacit}. It manifests as community-wide conventions---such as training tricks and high-level methodological insights---and can only be acquired by induction across large-scale implementations~\citep{shull2002replicating,10.1145/3736731.3746141}. Notably, a paper's local neighborhood alone is insufficient for recovering such knowledge, since its restricted receptive field~\citep{li2018deeper} cannot provide the generalizable and transferable insights that collective practices entail.

Such tacit knowledge cannot be acquired by retrieval-augmented methods alone, as it demands reasoning beyond direct retrieval~\citep{amiraz-etal-2025-distracting,fang2024trace,sun2025enhancing}.

Building on these insights, we propose \method, a graph-based agent framework for automated paper reproduction. Our main contributions are:
\begin{itemize}[leftmargin=*]
\item \textbf{Tacit Knowledge Formulation.} We cast automated paper reproduction as the progressive recovery of tacit knowledge inherent in scholarly writing and delineate three constituent categories---Relational, Somatic, and Collective---ranked by acquisition difficulty.
\item \textbf{Graph-Based Agent Framework.} We propose \method, which leverages citation-graph structure to recover each type of tacit knowledge through a dedicated mechanism: node-level relation-aware aggregation for relational knowledge, execution-feedback refinement for somatic knowledge, and graph-level knowledge induction for collective knowledge.
\item \textbf{Extended Benchmark.} We extend ReproduceBench to 10 tasks spanning 3 domains, comprising 191 papers for training, 30 for validation, and 40 for testing, establishing a more comprehensive benchmark for automated reproduction.
\item \textbf{State-of-the-Art Results.} \method{} achieves an average performance gap of 10.04\% against official implementations, improving over the strongest baseline by 24.68\%, with consistent gains across all three domains.
\end{itemize}

\section{Related Work}

\paragraph{\textbf{LLM-Driven Automated Paper Reproduction.}}
Recent work has framed paper-to-code reproduction as an end-to-end programming task using LLM-based agents~\citep{seo2025paper2codeautomatingcodegeneration,li2025deepcodeopenagenticcoding,zhao2025autoreproduceautomaticaiexperiment}. Paper2Code~\citep{seo2025paper2codeautomatingcodegeneration} employs a multi-agent pipeline (planning, analysis, generation) to produce dependency-consistent, executable code repositories. AutoReproduce~\citep{zhao2025autoreproduceautomaticaiexperiment} introduces a paper-lineage mechanism to retrieve text and code snippets from cited references, and enhances executability via unit-test-driven development. DeepCode~\citep{li2025deepcodeopenagenticcoding} alleviates context bottlenecks through compressed indexing, retrieval-based knowledge injection, and closed-loop verification. However, Collins' empirical analysis of scientific reproduction~\citep{collins2019tacit} identifies three progressively harder categories of tacit knowledge: (1)~relational tacit knowledge---implementation details omitted when authors assume familiarity with related work; (2)~somatic tacit knowledge---knowledge rooted in practitioners' experience and intuition that is difficult to articulate explicitly; and (3)~collective tacit knowledge---knowledge embedded in the collective practices of a research community. Effectively recovering these forms of tacit knowledge is key to improving reproduction quality and reliability.

\paragraph{\textbf{Graph-Augmented LLM Agents.}}
Graph structures have recently been introduced as explicit scaffolds for LLM agent reasoning, such as graph-augmented RAG~\citep{edge2025localglobalgraphrag,hu2025graggraphretrievalaugmentedgeneration} and Graph of Thoughts~\citep{Besta_2024}. Inspired by the well-established effectiveness of multi-scale reasoning in graph learning~\citep{xing2024moreoverglobalizingproblemgraph, zhangrestricted}---where node-level neighborhood aggregation captures local information and graph-level global analysis overcomes the restricted local receptive field to extract generalizable, transferable patterns---recent work has begun incorporating this paradigm into LLM agents to improve performance~\citep{zhang2025planovergraphparallelablellmagent}. However, how to effectively leverage multi-scale graph reasoning to recover tacit scientific knowledge from papers and thereby improve automated reproduction remains an open problem.

\begin{figure*}[t]
  \centering
  \includegraphics[width=\linewidth]{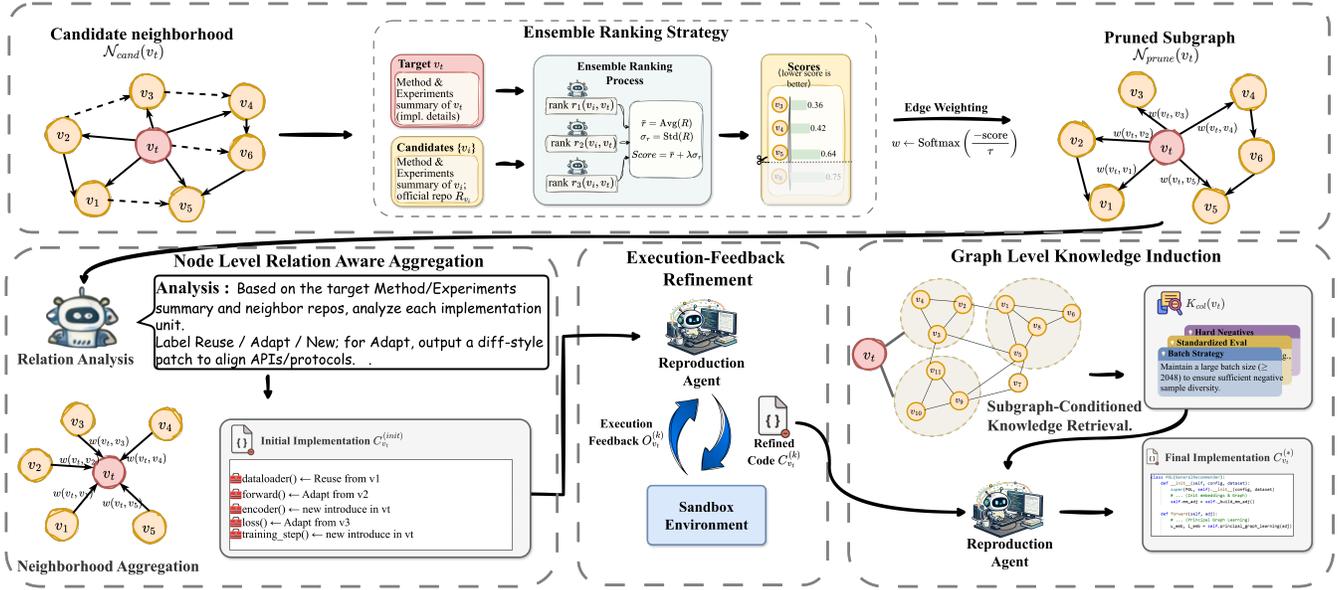}
  \caption{\textbf{Overview of \method.} The framework comprises three stages: SSGP prunes the citation graph to retain implementation-relevant neighbors; node-level relation-aware aggregation analyzes reuse/adapt/new relationships and assembles an initial implementation; execution-feedback refinement and graph-level knowledge induction iteratively improve it via runtime debugging and community-induced knowledge.}
  \label{fig:pipeline}
\end{figure*}

\section{Preliminaries}
\label{sec:preliminaries}

\subsection{Problem Definition}
\label{subsec:problem_def}

Each paper $v$ is represented as a tuple $v = (\mathcal{D}_v, \mathcal{C}_v)$, where $\mathcal{D}_v$ denotes its document and $\mathcal{C}_v$ the corresponding code implementation. Given a target paper $v_t$, automated paper reproduction seeks to generate $\hat{\mathcal{C}}_{v_t}$ that minimizes the reproduction quality loss:
\begin{equation}
\min_{\hat{\mathcal{C}}_{v_t}} \; \mathcal{L}\bigl(\mathcal{P}(\hat{\mathcal{C}}_{v_t}),\; \mathcal{P}(\mathcal{C}_{v_t})\bigr),
\end{equation}
where $\mathcal{P}(\cdot)$ evaluates the execution performance of a code implementation and $\mathcal{L}$ measures the discrepancy between the generated and official results.

\subsection{Scientific Graph}
\label{subsec:skg}

We model scientific literature as a scientific graph $\mathcal{G} = (\mathcal{V}, \mathcal{E})$, where the node set $\mathcal{V}$ represents the collection of papers, with each node $v \in \mathcal{V}$ corresponding to a paper and its attributes $(\mathcal{D}_v, \mathcal{C}_v)$ as defined in Section~\ref{subsec:problem_def}. The edge set $\mathcal{E} \subseteq \mathcal{V} \times \mathcal{V}$ represents citation relationships between papers, where each edge $(v, u) \in \mathcal{E}$ indicates that paper $v$ directly cites paper $u$. For any paper $v$, we define its neighbor set as $\mathcal{N}(v) = \{u \in \mathcal{V} : (v, u) \in \mathcal{E}\}$.

\subsection{Definition of Tacit Knowledge}
\label{subsec:tacit_knowledge}

\paragraph{\textbf{Knowledge Acquisition Difficulty.}}
Let $\Phi_{v_t}$ denote the set of implementation decisions required to reproduce $\mathcal{C}_{v_t}$, and let $\mathcal{R}$ denote the available resources accessible during reproduction. Given an agent $\mathcal{A}$, we write $\mathcal{R} \vdash_{\mathcal{A}} \phi$ if $\mathcal{A}$ can correctly determine decision $\phi$ under $\mathcal{R}$, and assume $\vdash_{\mathcal{A}}$ is monotone with respect to $\mathcal{R}$. Let $\mathcal{O}_v$ denote the execution feedback observed during the reproduction of paper $v$ (including training logs, error messages, and evaluation metrics). We define a hierarchy of progressively enriched resources:
\begin{equation}
\begin{aligned}
\mathcal{R}_0 &= \{\mathcal{D}_{v_t}\}, \\
\mathcal{R}_1 &= \mathcal{R}_0 \;\cup\; \{(\mathcal{D}_{v_i}, \mathcal{C}_{v_i})\}_{v_i \in \mathcal{N}(v_t)}, \\
\mathcal{R}_2 &= \mathcal{R}_1 \;\cup\; \{\mathcal{O}_{v_t}\}, \\
\mathcal{R}_3 &= \mathcal{R}_2 \;\cup\; \{(\mathcal{D}_v, \mathcal{C}_v)\}_{v \in \mathcal{S}^{(v_t)}} \;\cup\; \{\mathcal{O}_v\}_{v \in \mathcal{S}^{(v_t)}}.
\end{aligned}
\end{equation}
Here $\mathcal{N}(v_t)$ denotes the citation neighbors of $v_t$, while $\mathcal{S}^{(v_t)}$ denotes papers that share collective implementation practices with $v_t$, covering a broader scope than the local neighborhood. Each level successively introduces neighbor documents and code, the execution feedback of reproducing the target paper, and community-wide implementations along with their execution feedback, reflecting increasing difficulty of knowledge acquisition.

\paragraph{\textbf{Tacit Knowledge.}}
Tacit knowledge is defined as the set of decisions that cannot be determined from the paper document alone but become recoverable given the full resources:
\begin{equation}
\mathcal{K}(v_t) = \{\phi \in \Phi_{v_t} \mid \mathcal{R}_0 \nvdash_{\mathcal{A}} \phi,\; \mathcal{R}_3 \vdash_{\mathcal{A}} \phi\},
\end{equation}
and is decomposed into three categories according to acquisition difficulty:

\begin{itemize}[leftmargin=*]
\item \emph{Relational Tacit Knowledge}: decisions recoverable once neighbor code is available, e.g., module reuse and adaptation relationships omitted across papers. $\mathcal{K}_{\mathrm{rel}}(v_t) = \{\phi \in \Phi_{v_t} \mid \mathcal{R}_0 \nvdash_{\mathcal{A}} \phi,\; \mathcal{R}_1 \vdash_{\mathcal{A}} \phi\}$.
\item \emph{Somatic Tacit Knowledge}: decisions determinable only through code execution and observation of execution feedback, e.g., judging implementation correctness from training logs. $\mathcal{K}_{\mathrm{som}}(v_t) = \{\phi \in \Phi_{v_t} \mid \mathcal{R}_1 \nvdash_{\mathcal{A}} \phi,\; \mathcal{R}_2 \vdash_{\mathcal{A}} \phi\}$.
\item \emph{Collective Tacit Knowledge}: decisions obtainable only by induction over community-wide implementations, e.g., standard preprocessing conventions. $\mathcal{K}_{\mathrm{col}}(v_t) = \{\phi \in \Phi_{v_t} \mid \mathcal{R}_2 \nvdash_{\mathcal{A}} \phi,\; \mathcal{R}_3 \vdash_{\mathcal{A}} \phi\}$.
\end{itemize}

\begin{figure*}[t]
  \centering
  \includegraphics[width=1\linewidth]{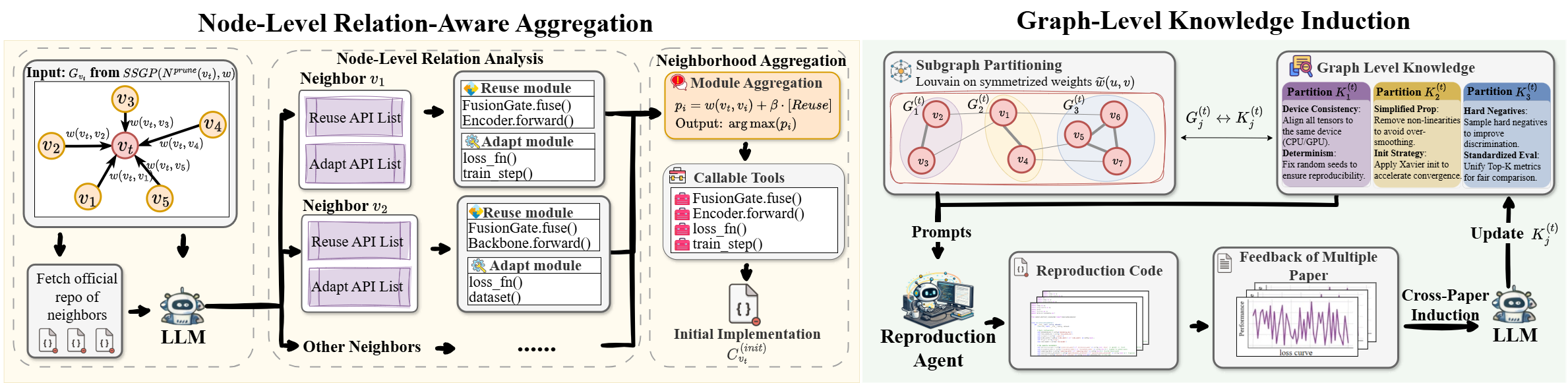}
  \caption{\textbf{Node-Level and Graph-Level Graph Reasoning.} \textbf{Left (Node-Level Relation-Aware Aggregation)}: For each neighbor $v_i \in \Nprune(v_t)$, relation analysis produces a structured annotation $\Gamma_{t,i}$ that categorizes implementation units into directly reusable ($\mathcal{U}^{\mathrm{reuse}}$), adaptable ($\mathcal{U}^{\mathrm{adapt}}$), and new ($\mathcal{U}^{\mathrm{new}}$); neighborhood aggregation selects among competing candidates by priority $p = w(v_t, v_i) + \beta \cdot \mathbb{1}[\text{reuse}]$ to construct $\mathcal{C}_{v_t}^{(\mathrm{init})}$. \textbf{Right (Graph-Level Knowledge Induction)}: The paper graph is partitioned into subgraphs $\{\mathcal{G}^{(t)}_j\}$ via Louvain clustering on SSGP edge weights; within each subgraph, the agent reproduces member papers, performs cross-paper induction over execution feedback, and writes recurring patterns into the subgraph knowledge base $\mathcal{K}^{(t)}_j$.}
  \label{fig:graphreasoning}
\end{figure*}

\section{The Proposed Method}

As illustrated in Figure~\ref{fig:pipeline}, \method{} operates on the scientific graph $\mathcal{G}$ in two stages. First, Semantic Scientific Graph Pruning (SSGP, Section~\ref{sec:semantic_pruning}) filters the citation neighborhood of the target paper $v_t$ to retain only implementation-relevant neighbors, yielding a pruned weighted subgraph $\Nprune(v_t)$. Second, Graph-Based Tacit Knowledge Recovery (Section~\ref{subsec:graph_tk_recovery}) progressively recovers the three types of tacit knowledge at complementary granularities: at the \emph{node level}, relation-aware aggregation analyzes reuse/adapt/new relationships between $v_t$ and each neighbor to construct an initial implementation $\mathcal{C}_{v_t}^{(\mathrm{init})}$; execution-feedback refinement then iteratively debugs it into $\mathcal{C}_{v_t}^{(\mathrm{ref})}$; at the \emph{graph level}, knowledge induction partitions the paper graph into implementation-coherent subgraphs and extracts community-shared practices into subgraph knowledge bases $\mathcal{K}^{(t)}_j$, which are retrieved to produce the final implementation $\mathcal{C}_{v_t}^{(*)}$.

\subsection{Semantic Scientific Graph Pruning}
\label{sec:semantic_pruning}
Not all citation neighbors are equally relevant to reproduction; to filter semantic noise, we construct a pruned weighted subgraph that retains only implementation-relevant neighbors.

\paragraph{\textbf{Semantic Relevance Measurement.}}
Following ~\citep{zhao2025autoreproduceautomaticaiexperiment}, we retain citation neighbors of $v_t$ with publicly available implementations to form the candidate neighborhood~$\Ncand(v_t)$. To quantify relevance, we construct a summary $S(\cdot)$ for each paper from its Method and Experiments sections. We adopt an ensemble ranking strategy: $K$ independent LLM reviewers each rank all candidates $u \in \Ncand(v_t)$ by implementation relevance to $v_t$. Let $r_k(u, v_t)$ denote the rank assigned to candidate $u$ by the $k$-th reviewer. We aggregate results into the mean rank $\bar{r}(u, v_t) = \frac{1}{K} \sum_{k=1}^{K} r_k(u, v_t)$ and rank standard deviation $\sigma_r(u, v_t)$. The composite score balances relevance and reviewer consensus:
\begin{equation}
\mathrm{score}(u, v_t) = \bar{r}(u, v_t) + \lambda \cdot \sigma_r(u, v_t),
\end{equation}
where $\lambda$ is the uncertainty penalty coefficient.

\paragraph{\textbf{Pruned Graph Construction.}}
Based on the composite scores, we select the $K_{\text{keep}}$ candidates with lowest scores to obtain the pruned neighborhood:
\begin{equation}
\Nprune(v_t) = \argmin_{\substack{S \subseteq \Ncand(v_t),\\ |S| = K_{\text{keep}}}} \sum_{u \in S} \mathrm{score}(u, v_t).
\end{equation}
To support subsequent graph reasoning, we map the composite scores to edge weights, obtaining a normalized weight distribution from $v_t$ to each retained neighbor. Specifically, we apply softmax to convert scores into normalized weights:
\begin{equation}
w(v_t, u) = \frac{\exp\bigl(-\mathrm{score}(u, v_t)\bigr)}{\sum_{u' \in \Nprune(v_t)} \exp\bigl(-\mathrm{score}(u', v_t)\bigr)}.
\end{equation}

\subsection{Graph-Based Tacit Knowledge Recovery}
\label{subsec:graph_tk_recovery}

Based on the pruned neighborhood $\Nprune(v_t)$ obtained from SSGP, we progressively recover the three types of tacit knowledge defined in Section~\ref{subsec:tacit_knowledge}. Node-Level Relation-Aware Aggregation recovers relational tacit knowledge $\mathcal{K}_{\mathrm{rel}}$ by analyzing implementation relationships between the target and its neighbors to produce an initial implementation. Execution-Feedback Refinement recovers somatic tacit knowledge $\mathcal{K}_{\mathrm{som}}$ through iterative debugging driven by execution feedback. Graph-Level Knowledge Induction recovers collective tacit knowledge $\mathcal{K}_{\mathrm{col}}$ by inducing community-shared practices from cross-paper reproduction at the subgraph level.

\subsubsection{Node-Level Relation-Aware Aggregation}
\label{subsubsec:node-level}

As illustrated in Figure \ref{fig:graphreasoning} (left), this module aims to recover implementation-unit-level relationships between $v_t$ and its neighbors under resource level $\mathcal{R}_1$ (Section~\ref{subsec:tacit_knowledge}), thereby providing an executable initial scaffold. However, a straightforward aggregation of neighboring code snippets yields only fragmentary signals, insufficient to systematically characterize how $\mathcal{C}_{v_t}$ reuses, adapts, or extends implementations from its neighbors. We therefore propose Node-Level Relation-Aware Aggregation, which first analyzes implementation relationships between the target paper and each neighbor node, and then aggregates neighbor implementations via the paper graph to construct an initial implementation $\mathcal{C}_{v_t}^{(\mathrm{init})}$.

\paragraph{\textbf{Node-Level Relation Analysis.}}
For each pruned neighbor $v_i \in \Nprune(v_t)$, this step takes $(\mathcal{D}_{v_t}, \mathcal{D}_{v_i}, \mathcal{C}_{v_i})$ as input, where $\mathcal{D}_v$ denotes the paper document and $\mathcal{C}_v$ the corresponding implementation. We refer to each functional component required to reproduce the target paper (e.g., a particular loss function, sampling strategy, or encoder) as an \emph{implementation unit}~\cite{luo2025executable}. The step outputs a structured relation annotation $\Gamma_{t,i}$ between $v_t$ and $v_i$ at the implementation-unit level:
\begin{equation}
\Gamma_{t,i} = \bigl(\mathcal{U}^{\mathrm{reuse}}_{t,i},\, \mathcal{U}^{\mathrm{adapt}}_{t,i},\, \mathcal{U}^{\mathrm{new}}_{t,i},\, \{\Delta_{t,i}^{m}\}_{m \in \mathcal{U}^{\mathrm{adapt}}_{t,i}}\bigr),
\end{equation}
where $\mathcal{U}^{\mathrm{reuse}}_{t,i}$ is the set of implementation units from $v_i$ that are directly reusable, $\mathcal{U}^{\mathrm{adapt}}_{t,i}$ is the set of units that must be adapted from $v_i$ to match $v_t$, and $\mathcal{U}^{\mathrm{new}}_{t,i}$ is the set of units lacking a direct counterpart in $v_i$ and thus requiring new implementation. For each $m \in \mathcal{U}^{\mathrm{adapt}}_{t,i}$, $\Delta_{t,i}^{m}$ specifies an actionable diff instruction~\cite{bouzenia2024repairagent}. Based on $\Gamma_{t,i}$, we further encapsulate these implementation units into a set of callable APIs $\mathcal{API}_{t,i}$ that the reproduction agent can directly invoke: for $u \in \mathcal{U}^{\mathrm{reuse}}_{t,i}$, the corresponding implementation is located and extracted as a directly reusable API; for $u \in \mathcal{U}^{\mathrm{adapt}}_{t,i}$, the original implementation is extracted and an adapted version is generated according to $\Delta_{t,i}^{u}$; for $u \in \mathcal{U}^{\mathrm{new}}_{t,i}$, only a placeholder stub is generated for subsequent completion. All encapsulated APIs undergo callability validation (dependency resolution, import checking, and interface-level smoke testing) to ensure they can be reliably invoked by the agent.

\paragraph{\textbf{Neighborhood Aggregation.}}
This stage takes as input the candidate API sets $\{\mathcal{API}_{t,i}\}_{v_i \in \Nprune(v_t)}$ from all neighbors together with the edge weights $\{w(v_t, v_i)\}$, and outputs the initial implementation $\mathcal{C}_{v_t}^{(\mathrm{init})}$. Since different neighbors may provide competing implementations for the same implementation unit $u$, we select among them based on graph edge weights. Specifically, let $\{api_i^{u}\}$ denote the candidate implementations from different neighbors (where $api_i^{u} \in \mathcal{API}_{t,i}$ and is labeled as reuse or adapt). We compute a priority score for each candidate:
\begin{equation}
p(api_i^{m}) = w(v_t, v_i) + \beta \cdot \mathbb{1}[api_i^{m}\ \text{is reuse}],
\end{equation}
where $\beta > 0$ is a reuse bias coefficient that favors direct reuse over adapted implementations when relevance scores are comparable, thereby reducing the risk of implementation drift introduced by code transformation. For each unit $u$, the candidate with the highest $p(\cdot)$ is selected; units covered only by $\mathcal{U}^{\mathrm{new}}_{t,i}$ placeholders are deferred to the subsequent Execution-Feedback Refinement stage. The resulting $\mathcal{C}_{v_t}^{(\mathrm{init})}$ serves as the initial implementation aggregated from the neighborhood.

\subsubsection{Execution-Feedback Refinement}
\label{subsubsec:exec-feedback}

The initial implementation $\mathcal{C}_{v_t}^{(\mathrm{init})}$ may still fail to execute correctly or exhibit performance deviations due to implicit experiential decisions---such as tensor dimension conventions, numerical stability handling, and training schedule configurations. These decisions correspond to somatic tacit knowledge $\mathcal{K}_{\mathrm{som}}(v_t)$, which can only be localized by observing execution feedback $\mathcal{O}_{v_t}$. Following ReAct~\citep{yao2023reactsynergizingreasoningacting} for iterative refinement, we augment diagnostic signals with execution performance discrepancies relative to the paper-reported results.

\paragraph{\textbf{Iterative Refinement.}}
This module takes as input the paper document, the initial implementation, and its execution feedback, i.e., $(\mathcal{D}_{v_t}, \mathcal{C}_{v_t}^{(\mathrm{init})}, \mathcal{O}_{v_t})$, and outputs a refined implementation $\mathcal{C}_{v_t}^{(\mathrm{ref})}$ together with the final execution log. Following ReAct, at the $k$-th iteration the current implementation $\mathcal{C}_{v_t}^{(k)}$ is executed for training and evaluation, yielding execution feedback $\mathcal{O}_{v_t}^{(k)}$---including error messages, logs, and performance metrics $\mathcal{P}(\mathcal{C}_{v_t}^{(k)})$. The execution feedback $\mathcal{O}_{v_t}^{(k)}$, together with $\mathcal{D}_{v_t}$ and the current implementation, is then fed to the reproduction agent, which generates an executable repair plan:
\begin{equation}
\mathrm{Plan}^{(k)} = \bigl(\mathcal{U}^{(k)}_{\mathrm{edit}},\; \{\delta^{u,(k)}\}_{u \in \mathcal{U}^{(k)}_{\mathrm{edit}}}\bigr),
\end{equation}
where $\mathcal{U}^{(k)}_{\mathrm{edit}}$ is the set of implementation units to be modified and $\delta^{u,(k)}$ is the corresponding code change. Applying $\mathrm{Plan}^{(k)}$ to $\mathcal{C}_{v_t}^{(k)}$ yields $\mathcal{C}_{v_t}^{(k+1)}$. This loop continues until the reproduction loss $\mathcal{L}(\mathcal{P}(\mathcal{C}_{v_t}^{(k)}), \mathcal{P}(\mathcal{C}_{v_t}))$ drops below a predefined threshold or a maximum iteration budget is exhausted. The final output $\mathcal{C}_{v_t}^{(\mathrm{ref})}$ serves as the base implementation for subsequent graph-level knowledge induction.

\subsubsection{Graph-Level Knowledge Induction}
\label{subsubsec:graph-level}

Through Sections~\ref{subsubsec:node-level} and~\ref{subsubsec:exec-feedback}, we recover $\mathcal{K}_{\mathrm{rel}}$ and $\mathcal{K}_{\mathrm{som}}$ to obtain the refined implementation $\mathcal{C}_{v_t}^{(\mathrm{ref})}$. However, transferable knowledge embedded in the collective practices of a research community---such as typical hyperparameter ranges, training tricks---is difficult to acquire reliably from the execution feedback of a single paper and requires induction over large-scale cross-paper implementation experience. Such knowledge corresponds to collective tacit knowledge $\mathcal{K}_{\mathrm{col}}$, recoverable only at resource level $\mathcal{R}_3$. A naive strategy is to aggregate experience from randomly sampled papers, yet a single domain often hosts competing implementation paradigms (e.g., MLP-based vs.\ Transformer-based approaches in time series forecasting)~\cite{zeng2023transformers}. Inducing from papers without tight implementation associations makes it difficult to extract generalizable, transferable knowledge. Therefore, we leverage the implementation similarity encoded by edge weights to partition the paper graph into subgraphs whose members share tight implementation associations, and then induce $\mathcal{K}_{\mathrm{col}}$ within each subgraph from reproduction outcomes. The approach is shown in Figure \ref{fig:graphreasoning} (right).

\paragraph{\textbf{Subgraph Partition.}}
For the task-level paper graph $\mathcal{G}^{(t)} = (\mathcal{S}^{(v_t)}, \mathcal{E}^{(t)})$, the SSGP-derived edge weights $w(\cdot, \cdot)$ capture the semantic similarity in method implementation between papers. We symmetrize the directed weights as $\tilde{w}(u, v) = \frac{1}{2}\bigl(w(u, v) + w(v, u)\bigr)$ and apply the Louvain algorithm~\citep{blondel2008fast} on the weighted undirected graph $(\mathcal{S}^{(v_t)}, \tilde{w})$ to obtain a subgraph partition $\{\mathcal{G}^{(t)}_j\}_{j=1}^{M_t}$. Papers within the same subgraph share tight implementation associations, and subsequent induction is performed only within each subgraph, thereby restricting the observation scope of $\mathcal{R}_3$ to paradigm-coherent paper collections.

\paragraph{\textbf{Subgraph-Level Knowledge Induction.}}
Each subgraph $\mathcal{G}^{(t)}_j$ maintains an independent knowledge base $\mathcal{K}^{(t)}_j$. At each training epoch, the agent reproduces all papers $v \in \mathcal{G}^{(t)}_j$ via the pipeline of Sections~\ref{subsubsec:node-level} and~\ref{subsubsec:exec-feedback}, obtaining each paper's refined implementation $\mathcal{C}_v^{(\mathrm{ref})}$ and execution feedback $\mathcal{O}_v$. Cross-paper induction is then performed over the collected $\{(\mathcal{C}_v^{(\mathrm{ref})}, \mathcal{O}_v)\}_{v \in \mathcal{G}^{(t)}_j}$: an LLM identifies recurring patterns within the subgraph, retaining only entries whose frequency exceeds threshold $\eta$ and that yield stable reproduction gains on the validation set $\mathcal{S}^{(v_t)}_{\mathrm{val}}$, and writes them into $\mathcal{K}^{(t)}_j$ in structured form. The optimal knowledge base for each subgraph, $\{\mathcal{K}^{(t,*)}_j\}_{j=1}^{M_t}$, is saved.

\paragraph{\textbf{Subgraph-Conditioned Knowledge Retrieval.}}
At inference time, given a target paper $v_t$, we compute its affinity $s_j$ with each subgraph using the edge weights:
\begin{equation}
s_j = \sum_{u \in \Nprune(v_t) \cap \mathcal{G}^{(t)}_j} w(v_t, u).
\end{equation}
We select the top-$K$ subgraph index set $\mathcal{J}^*$ and concatenate the corresponding knowledge:
\begin{equation}
\mathcal{K}_{\mathrm{col}}(v_t) = \bigcup_{j \in \mathcal{J}^*} \mathcal{K}^{(t,*)}_j.
\end{equation}
Building upon $\mathcal{C}_{v_t}^{(\mathrm{ref})}$, we inject $\mathcal{K}_{\mathrm{col}}(v_t)$ as additional context into the reproduction agent to further improve reproduction performance, yielding the final implementation $\mathcal{C}_{v_t}^{(*)}$.

\subsubsection{Full Pipeline}
\label{subsubsec:full-pipeline}

Given a target paper $v_t$, Node-Level Relation-Aware Aggregation (Section~\ref{subsubsec:node-level}) aggregates neighbor implementations into $\mathcal{C}_{v_t}^{(\mathrm{init})}$. Execution-Feedback Refinement (Section~\ref{subsubsec:exec-feedback}) then iteratively refines it into $\mathcal{C}_{v_t}^{(\mathrm{ref})}$. Meanwhile, Graph-Level Knowledge Induction (Section~\ref{subsubsec:graph-level}) partitions the paper graph into subgraphs and induces $\mathcal{K}^{(t,*)}_j$ within each. Finally, the retrieved $\mathcal{K}_{\mathrm{col}}(v_t)$ is injected into the reproduction agent to further improve $\mathcal{C}_{v_t}^{(\mathrm{ref})}$, producing the final implementation $\mathcal{C}_{v_t}^{(*)}$.

\section{Experiments}

\subsection{Experimental Setup}
\label{subsec:exp_setup}

\begin{table*}[t]
    \centering
    \caption{Performance Gap ($\downarrow$) across three domains and ten task types. Lower values indicate better reproduction quality. Results are averaged over 5 independent runs; we report mean $\pm$ std.}
    \label{tab:cross_domain_all}
    \small
    \setlength{\tabcolsep}{3.5pt}
    \renewcommand{\arraystretch}{1.20}
    
    \begin{tabular}{llccccccc}
    \toprule
    \multicolumn{1}{c}{\textbf{Domain}} & \textbf{Task} & \textbf{ReAct} & \textbf{OpenHands} & \textbf{Paper2Code} & \textbf{DeepCode} & \textbf{AutoReproduce} & \textbf{Ours} \\
    \midrule
    
    \cellcolor{DomRec} & \cellcolor{DomRec}MMRec & \cellcolor{DomRec}81.66{\scriptsize$\pm$12.3} & \cellcolor{DomRec}67.44{\scriptsize$\pm$14.7} & \cellcolor{DomRec}70.13{\scriptsize$\pm$15.2} & \cellcolor{DomRec}64.58{\scriptsize$\pm$13.4} & \cellcolor{DomRec}36.42{\scriptsize$\pm$8.7} & \cellcolor{DomRec}\best{20.17}{\scriptsize$\pm$4.6} \\
    \cellcolor{DomRec} & \cellcolor{DomRec}GeneralRec & \cellcolor{DomRec}77.85{\scriptsize$\pm$14.1} & \cellcolor{DomRec}58.77{\scriptsize$\pm$11.6} & \cellcolor{DomRec}32.11{\scriptsize$\pm$12.4} & \cellcolor{DomRec}28.66{\scriptsize$\pm$10.8} & \cellcolor{DomRec}28.34{\scriptsize$\pm$7.2} & \cellcolor{DomRec}\best{8.77}{\scriptsize$\pm$2.4} \\
    \cellcolor{DomRec}\multirow{-3}{*}{\parbox{2cm}{\centering\textbf{RecSys}}} & \cellcolor{DomRec}SeqRec & \cellcolor{DomRec}80.33{\scriptsize$\pm$13.5} & \cellcolor{DomRec}71.09{\scriptsize$\pm$16.2} & \cellcolor{DomRec}65.88{\scriptsize$\pm$14.8} & \cellcolor{DomRec}61.24{\scriptsize$\pm$12.1} & \cellcolor{DomRec}34.76{\scriptsize$\pm$9.1} & \cellcolor{DomRec}\best{3.38}{\scriptsize$\pm$1.9} \\
    
    \midrule
    
    \cellcolor{DomTS} & \cellcolor{DomTS}Classification & \cellcolor{DomTS}78.57{\scriptsize$\pm$15.4} & \cellcolor{DomTS}62.18{\scriptsize$\pm$12.9} & \cellcolor{DomTS}68.44{\scriptsize$\pm$16.1} & \cellcolor{DomTS}63.11{\scriptsize$\pm$13.7} & \cellcolor{DomTS}38.23{\scriptsize$\pm$8.4} & \cellcolor{DomTS}\best{5.03}{\scriptsize$\pm$2.6} \\
    \cellcolor{DomTS} & \cellcolor{DomTS}LongTerm & \cellcolor{DomTS}82.11{\scriptsize$\pm$11.8} & \cellcolor{DomTS}75.33{\scriptsize$\pm$14.3} & \cellcolor{DomTS}37.66{\scriptsize$\pm$10.5} & \cellcolor{DomTS}34.22{\scriptsize$\pm$9.2} & \cellcolor{DomTS}31.87{\scriptsize$\pm$7.6} & \cellcolor{DomTS}\best{3.27}{\scriptsize$\pm$1.7} \\
    \cellcolor{DomTS} & \cellcolor{DomTS}ShortTerm & \cellcolor{DomTS}76.88{\scriptsize$\pm$16.7} & \cellcolor{DomTS}55.44{\scriptsize$\pm$13.2} & \cellcolor{DomTS}64.79{\scriptsize$\pm$14.6} & \cellcolor{DomTS}60.13{\scriptsize$\pm$12.4} & \cellcolor{DomTS}35.67{\scriptsize$\pm$8.9} & \cellcolor{DomTS}\best{6.01}{\scriptsize$\pm$3.5} \\
    \cellcolor{DomTS}\multirow{-4}{*}{\parbox{2cm}{\centering\textbf{Time Series}}} & \cellcolor{DomTS}AnomalyDetection & \cellcolor{DomTS}80.66{\scriptsize$\pm$13.6} & \cellcolor{DomTS}73.55{\scriptsize$\pm$15.1} & \cellcolor{DomTS}67.22{\scriptsize$\pm$13.9} & \cellcolor{DomTS}62.88{\scriptsize$\pm$11.3} & \cellcolor{DomTS}40.12{\scriptsize$\pm$9.3} & \cellcolor{DomTS}\best{25.44}{\scriptsize$\pm$5.2} \\
    
    \midrule
    
    \cellcolor{DomGML} & \cellcolor{DomGML}GeneralGL & \cellcolor{DomGML}79.44{\scriptsize$\pm$14.2} & \cellcolor{DomGML}60.88{\scriptsize$\pm$11.7} & \cellcolor{DomGML}66.55{\scriptsize$\pm$15.3} & \cellcolor{DomGML}61.77{\scriptsize$\pm$12.6} & \cellcolor{DomGML}33.45{\scriptsize$\pm$8.2} & \cellcolor{DomGML}\best{4.66}{\scriptsize$\pm$2.1} \\
    \cellcolor{DomGML} & \cellcolor{DomGML}GSL & \cellcolor{DomGML}81.22{\scriptsize$\pm$11.4} & \cellcolor{DomGML}78.66{\scriptsize$\pm$17.5} & \cellcolor{DomGML}70.33{\scriptsize$\pm$13.8} & \cellcolor{DomGML}65.11{\scriptsize$\pm$11.9} & \cellcolor{DomGML}39.78{\scriptsize$\pm$9.5} & \cellcolor{DomGML}\best{17.86}{\scriptsize$\pm$4.9} \\
    \cellcolor{DomGML}\multirow{-3}{*}{\parbox{2cm}{\centering\textbf{Graph Learning}}} & \cellcolor{DomGML}NoisyGL & \cellcolor{DomGML}77.11{\scriptsize$\pm$15.6} & \cellcolor{DomGML}56.79{\scriptsize$\pm$10.4} & \cellcolor{DomGML}35.44{\scriptsize$\pm$11.2} & \cellcolor{DomGML}31.08{\scriptsize$\pm$9.7} & \cellcolor{DomGML}28.56{\scriptsize$\pm$7.8} & \cellcolor{DomGML}\best{5.79}{\scriptsize$\pm$2.9} \\
    
    \bottomrule
    \end{tabular}
    \end{table*}

\paragraph{\textbf{Benchmark.}}
We extend ReproduceBench~\citep{zhao2025autoreproduceautomaticaiexperiment}, which comprises only 13 papers with limited domain coverage and potential data contamination risks due to early publication dates. Our benchmark encompasses three research domains: \emph{recommendation systems}, \emph{time series analysis}, and \emph{graph learning}, spanning 10 tasks. Papers are stratified into training (191), validation (30), and test (40) sets by paper release date, with the test set exclusively containing papers released after 2025 to mitigate contamination. Each paper is evaluated on two representative datasets to ensure robustness. Further details are provided in Appendix~\ref{sec:benchmark}.

\paragraph{\textbf{Evaluation Metrics.}}
We adopt \emph{Performance Gap} from ReproduceBench~\citep{zhao2025autoreproduceautomaticaiexperiment} as our core metric:
\begin{equation}
\text{Performance Gap} = \frac{1}{n} \sum_{i=1}^{n} \frac{|\mathcal{P}(\mathcal{C}_{v_t})_i - \mathcal{P}(\hat{\mathcal{C}}_{v_t})_i|}{\max\bigl(\mathcal{P}(\mathcal{C}_{v_t})_i,\; \mathcal{P}(\hat{\mathcal{C}}_{v_t})_i\bigr)} \times 100\%
\end{equation}
where $\mathcal{P}(\mathcal{C}_{v_t})$ and $\mathcal{P}(\hat{\mathcal{C}}_{v_t})$ denote the execution performance of the official implementation and the agent-generated implementation, respectively. Lower values indicate better reproduction. For non-executable code, $\mathcal{P}(\hat{\mathcal{C}}_{v_t})$ is set to 0. Additionally, 3 domain experts per task independently evaluate the generated code across three dimensions—\emph{method fidelity}, \emph{experimental configuration consistency}, and \emph{code completeness}~\citep{zhao2025autoreproduceautomaticaiexperiment}—each on a 1--10 scale, with inter-rater reliability reported via Fleiss' Kappa.

\paragraph{\textbf{Baselines.}}
We compare against two categories of representative agents: \textbf{(i)~General coding agents}, including \textbf{ReAct}~\citep{yao2023reactsynergizingreasoningacting} and \textbf{OpenHands}~\citep{wang2025openhandsopenplatformai}; \textbf{(ii)~Reproduction-specific agents}, including \textbf{Paper2Code}~\citep{seo2025paper2codeautomatingcodegeneration}, \textbf{DeepCode}~\citep{li2025deepcodeopenagenticcoding}, and \textbf{AutoReproduce}~\citep{zhao2025autoreproduceautomaticaiexperiment}. Among them, AutoReproduce leverages a paper-lineage mechanism to retrieve implementation snippets from citation neighbors.

\paragraph{\textbf{Implementation Details.}}
All methods share the same backbone (Claude Opus 4.5)~\cite{anthropic2025claude}, hardware (64$\times$ NVIDIA A800 80GB), and software stack (PyTorch 2.7~\cite{steiner2019pytorch}, CUDA 12.6). Results are reported as mean $\pm$ std over 5 independent runs.
For our method, the key hyperparameters are set as follows: in SSGP, the number of ensemble LLM reviewers $K{=}5$ and the uncertainty penalty $\lambda{=}0.5$; in node-level aggregation, the reuse bias $\beta{=}0.15$; in graph-level knowledge induction, the knowledge frequency threshold $\eta{=}\lfloor 0.5 \cdot |\mathcal{G}^{(t)}_j| \rfloor$ and the number of retrieved subgraphs top-$K{=}3$. Training of the graph-level knowledge base runs for 3 epochs on the training set, with the best checkpoint selected on the validation set. Prompt details are provided in Appendix~\ref{app:prompts}.

\subsection{Main Results and Ablation Study}

\paragraph{\textbf{Main Results.}}
Table~\ref{tab:cross_domain_all} presents the cross-domain performance gap comparison. Our method achieves an average performance gap of 10.04\%, substantially outperforming the strongest baseline AutoReproduce. \textbf{(1)~Practicality:} Existing methods achieve a performance gap below 30\% in only 3 out of 50 evaluation instances (5 baselines $\times$ 10 tasks), failing to meet practical research requirements; in contrast, our method attains gaps below 26\% on all tasks and below 10\% on 7 out of 10 tasks. \textbf{(2)~Robustness:} Our method exhibits stronger cross-task generalization capability and lower variance across multiple runs; even on challenging tasks such as time series anomaly detection, it maintains leading performance. 

\begin{table}[t]
  \centering
  \small
  \vspace{-2mm}
  \caption{Human evaluation results ($\uparrow$). Three domain experts independently assess each reproduction across three dimensions (1--10 scale). We report mean and Fleiss' $\kappa$.}
  \vspace{-2mm}
  \label{tab:human_eval}
  \setlength{\tabcolsep}{3.5pt}
  \renewcommand{\arraystretch}{1.10}
  \begin{tabular}{l|ccc|c|c}
  \toprule
  \rowcolor{tableheader}
  & \textbf{Method} & \textbf{Param.} & \textbf{Exp.} & \textbf{Avg.} & \textbf{$\kappa$} \\
  \midrule
  Paper2Code & 4.67 & 4.23 & 5.12 & 4.67 & 0.71 \\
  DeepCode & 5.12 & 4.89 & 5.56 & 5.19 & 0.68 \\
  AutoReproduce & 6.23 & 5.78 & 6.67 & 6.23 & 0.72 \\
  \midrule
  \rowcolor{rowhighlight}
  \textbf{Ours} & \best{8.45} & \best{8.12} & \best{8.89} & \best{8.49} & 0.76 \\
  \bottomrule
  \end{tabular}
  \vspace{-4mm}
\end{table}

Table~\ref{tab:human_eval} presents the human evaluation results. Our method achieves the highest scores across all three dimensions. The inter-rater agreement (Fleiss' $\kappa = 0.76$) indicates substantial consistency among expert evaluators. Notably, while AutoReproduce shows competitive execution-based metrics, the gap in human evaluation scores (6.23 vs.\ 8.49) reveals that our approach produces code that better aligns with the methodological intent of the original papers.

\paragraph{\textbf{Ablation Study.}}
Table~\ref{tab:ablation} reports the results of removing each core component individually. Removing execution-feedback refinement causes the largest degradation, as somatic tacit knowledge $\mathcal{K}_{\mathrm{som}}$ cannot be recovered without iterative runtime observation. SSGP removal ranks second, since all downstream reasoning modules depend on the high-quality pruned subgraphs it constructs.

\begin{table}[t!]
  \centering
  \small
  \caption{Ablation study on key components. Performance Gap ($\downarrow$).}
  \label{tab:ablation}
  \setlength{\tabcolsep}{4pt}
  \renewcommand{\arraystretch}{1.15}
  \begin{tabular}{cccc|cc}
  \toprule
  \rowcolor{tableheader}
  \textbf{SSGP} & \textbf{Node} & \textbf{Exec.} & \textbf{Graph} & \textbf{RecSys} & \textbf{TimeSeries} \\
  \midrule
  $\circ$ & \ding{52} & \ding{52} & \ding{52} & 44.33 & 47.66 \\
  \ding{52} & $\circ$ & \ding{52} & \ding{52} & 21.77 & 25.44 \\
  \ding{52} & \ding{52} & $\circ$ & \ding{52} & 49.55 & 52.33 \\
  \ding{52} & \ding{52} & \ding{52} & $\circ$ & 36.88 & 38.66 \\
  \midrule
  \rowcolor{rowhighlight}
  \ding{52} & \ding{52} & \ding{52} & \ding{52} & \best{10.77} & \best{9.94} \\
  \bottomrule
  \end{tabular}
\end{table}

\subsection{Effectiveness of Graph-Level Knowledge Induction.}

\paragraph{\textbf{\ding{192}~Generality and Transferability of  $\mathcal{K}_{\mathrm{col}}$.}} We train the knowledge base on two backbone LLMs of different scales (Claude-Opus-4.5 and Qwen3-Next-80B-A3B-Thinking) and transfer it to four target models. As shown in Figure~\ref{fig:transfer}, regardless of the training backbone, the knowledge base yields consistent gains across all targets (average 9.4\%--13.2\%). Notably, the knowledge base trained on the smaller model still brings substantial improvements when transferred to larger ones. These results suggest that the induced knowledge is model-agnostic---precisely the defining property of $\mathcal{K}_{\mathrm{col}}$---confirming its generality and transferability.

\paragraph{\textbf{\ding{193}~Effectiveness of Subgraph Partitioning.}} We compare our edge-weight-based subgraph partitioning against random partitioning. As shown in Figure~\ref{fig:training_curve}, our partitioning strategy significantly outperforms random partitioning in both convergence speed and final performance. This confirms that SSGP edge weights capture inter-paper paradigm similarity, enabling paradigm-consistent papers to be grouped into coherent subgraphs $\mathcal{G}^{(t)}_j$ for more effective induction of $\mathcal{K}_{\mathrm{col}}$.

\paragraph{\textbf{\ding{194}~Case Study.}} Figure~\ref{fig:case_study} traces the knowledge base evolution on the multimodal recommendation task: epoch~1 yields surface-level engineering rules (e.g., device consistency); epoch~2, more sophisticated training tricks (e.g., graph normalization, sampling); epoch~3, high-level methodological insights (e.g., simplified propagation design). This progressive evolution---from engineering details to methodological insights---closely aligns with the defining characteristic of $\mathcal{K}_{\mathrm{col}}$, confirming that graph-level knowledge induction progressively distills collective tacit knowledge.

\begin{figure}[t]
  \centering
  \includegraphics[width=1\linewidth]{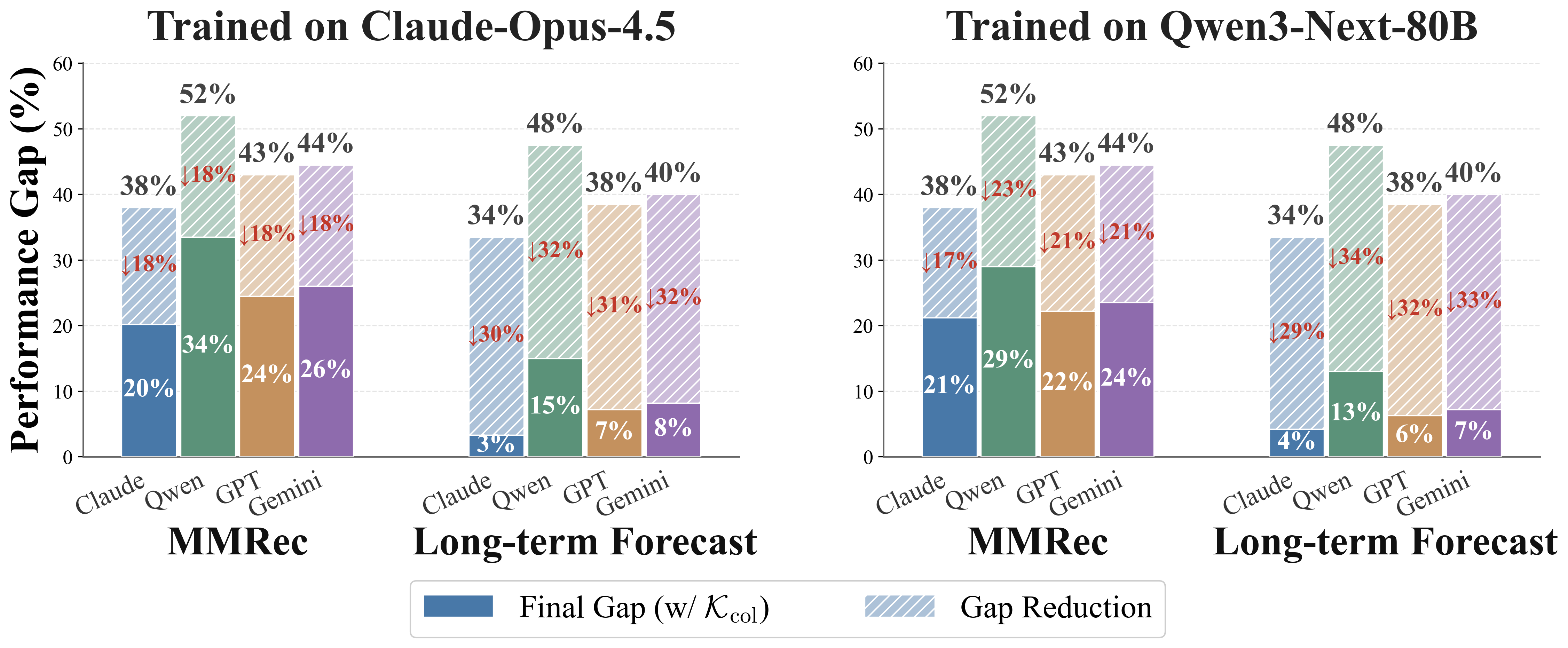}
  \caption{\textbf{Cross-model transferability of $\mathcal{K}_{\mathrm{col}}$.} The induced knowledge base yields consistent gains when transferred across different backbone LLMs.}
  \label{fig:transfer}
\end{figure}
\begin{figure}[t]
  \centering
  \includegraphics[width=1\linewidth]{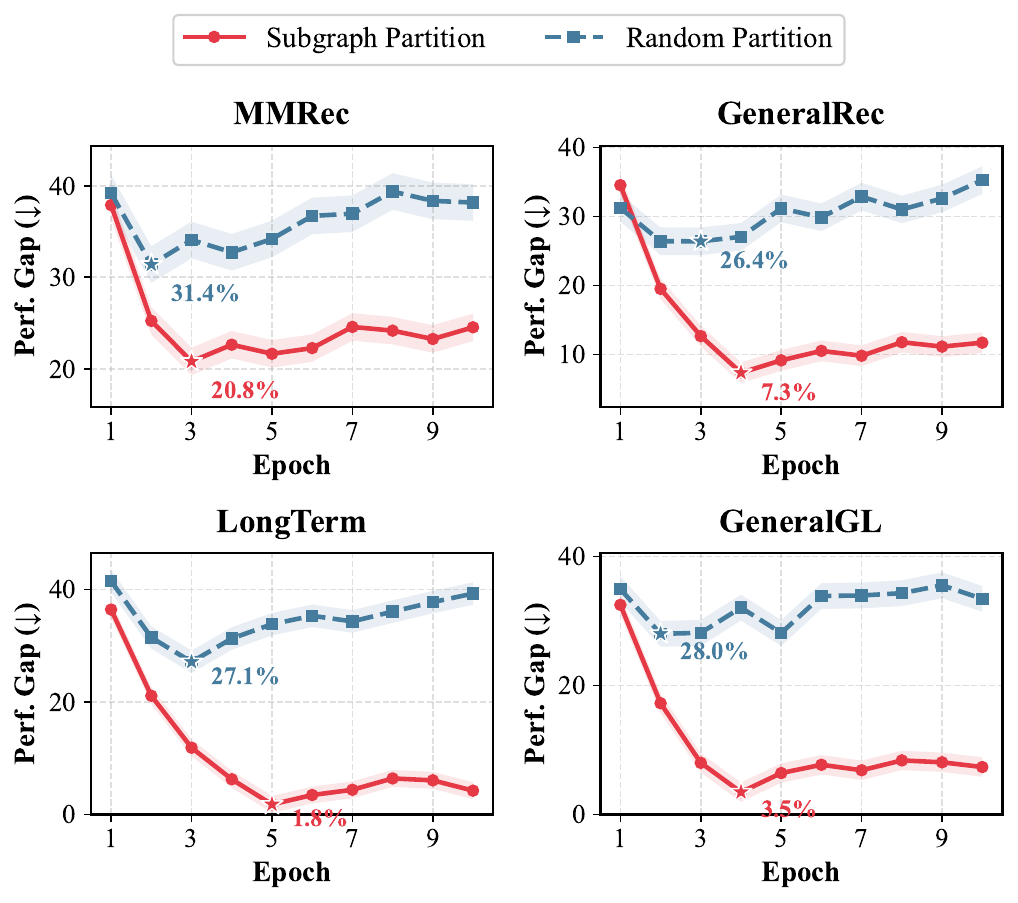}
  \caption{\textbf{Subgraph partitioning comparison.} Our edge-weight-based partitioning vs.\ random partitioning across training epochs.}
  \label{fig:training_curve}
\end{figure}

\begin{figure}[t]
  \centering
  \includegraphics[width=1\linewidth]{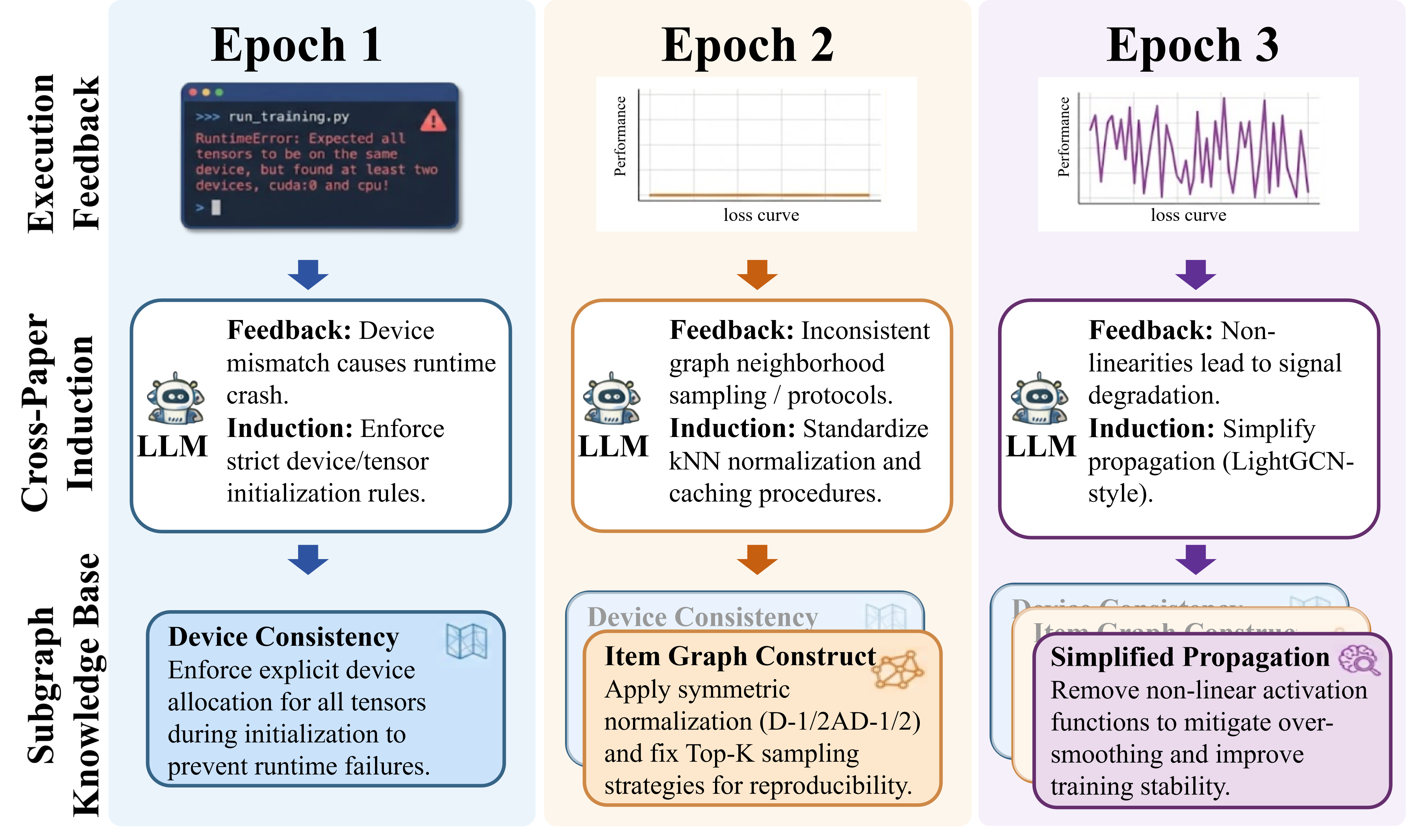}
  \caption{\textbf{Knowledge base evolution across training epochs.} Induced knowledge progresses from engineering rules to training tricks to methodological insights.}
  \label{fig:case_study}
\end{figure}

\vspace{-2mm}
\begin{figure}[t!]
\centering
\includegraphics[width=0.95\linewidth]{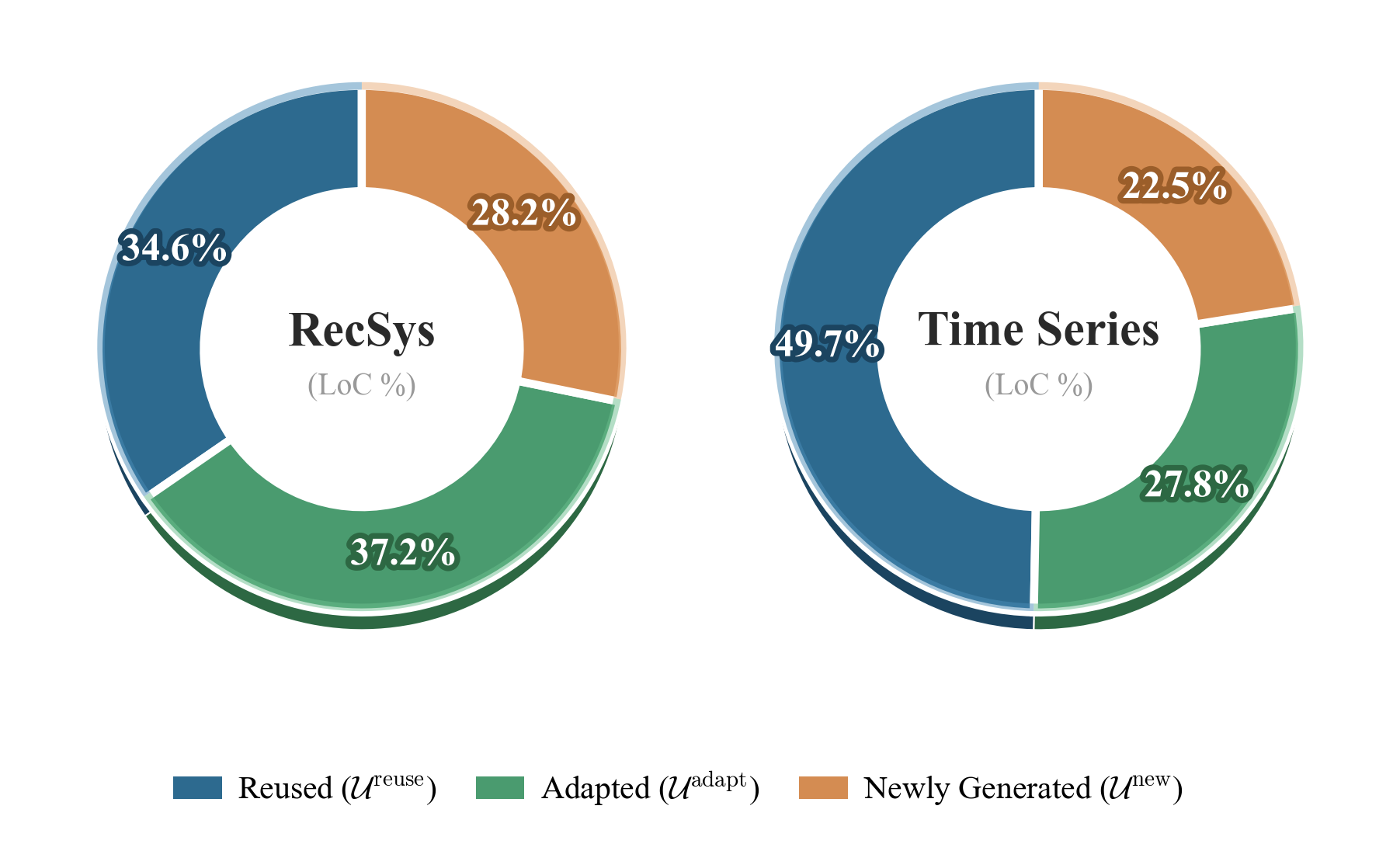}
\vspace{-2mm}
\caption{\textbf{Code-source breakdown.} Proportions of reused, adapted, and newly generated code across tasks.}
\vspace{-5mm}

\label{fig:code_source_analysis}
\end{figure}

\begin{figure}[t]
\centering
\begin{minipage}[t]{0.485\textwidth}
\begin{codeboxHL}{Reuse Candidate for Item Modality Graph Construction}
\begin{lstlisting}[style=paperpy]
# Reuse Candidate (Neighbor: FREEDOM)
def get_knn_adj_mat(self, mm_emb):
    ...
    mm_emb = F.normalize(mm_emb, dim=1)
    sim = mm_emb @ mm_emb.t()
    _, idx = torch.topk(sim, self.knn_k, dim=-1)
    ...
    adj = build_sparse_adj(idx, sim.size())
    adj = normalize_laplacian(adj)
    return adj
\end{lstlisting}
\end{codeboxHL}
\end{minipage}
\hfill
\vspace{0.5mm}
\begin{minipage}[t]{0.485\textwidth}
\begin{codeboxHL}{Adapt Candidate for Item Modality Graph Construction}
\begin{lstlisting}[style=paperdiff]
- def get_knn_adj_mat(self, img_emb, txt_emb):
-     ...
-     w = softmax(self.modal_weight)
-     sim = w[0]*sim(img_emb) + w[1]*sim(txt_emb)
-     return build_knn(sim, self.knn_k)

+ def get_knn_adj_mat(self, mm_emb):
+     ...
+     img_emb, txt_emb = mm_emb
+     w = softmax(self.modal_weight)          # keep LATTICE fusion
+     sim = w[0]*sim(img_emb) + w[1]*sim(txt_emb)
+     ...
+     adj = build_sparse_adj(topk(sim), sim.size())
+     adj = normalize_laplacian(adj)          # patched to static form
+     return adj
\end{lstlisting}
\end{codeboxHL}
\end{minipage}
\caption{\textbf{Neighborhood aggregation case study.} The priority score selects FREEDOM's reuse candidate over LATTICE's adapted variant for PGL's kNN graph construction.}
%\vspace{-15mm}

\label{fig:module_aggregation}
\end{figure}

\subsection{Effectiveness of Node-Level Relation-Aware Aggregation}

\paragraph{\textbf{\ding{192}~Effectiveness of Node-Level Relation Analysis.}}
As shown in Figure~\ref{fig:code_source_analysis}, We categorize the code in $\mathcal{C}_{v_t}^{(\mathrm{init})}$ by its origin: reused ($\mathcal{U}^{\mathrm{reuse}}$), adapted ($\mathcal{U}^{\mathrm{adapt}}$), and newly generated ($\mathcal{U}^{\mathrm{new}}$). As shown in Figure~\ref{fig:code_source_analysis}, in recommendation and time series tasks, reused code accounts for 34.6\% and 49.7\%, adapted code for 37.2\% and 27.8\%, while newly generated code constitutes only 28.2\% and 22.5\%. Over 70\% of the generated code is derived from neighbor implementations, confirming that node-level relation analysis accurately identifies reuse and adaptation relationships and substantially reduces code generated from scratch.

\paragraph{\textbf{\ding{193}~Effectiveness of Neighborhood Aggregation.}}
Figure~\ref{fig:module_aggregation} illustrates a representative case: for the item--item kNN graph construction in PGL, FREEDOM provides a reusable static implementation ($w{=}0.45$) and LATTICE an adapted dynamic variant ($w{=}0.37$). The priority score $p = w(v_t, v_i) + \beta \cdot \mathbb{1}[\text{reuse}]$ favors FREEDOM's static reuse, consistent with the official PGL implementation, validating the aggregation mechanism under multi-candidate competition.

%===========codebox==================

\section{Conclusion}
We present \method, a graph-based agent framework that formulates automated paper reproduction as the progressive recovery of tacit knowledge omitted in academic writing. By decomposing tacit knowledge into three types---relational, somatic, and collective---\method{} addresses each through a dedicated mechanism: node-level relation-aware aggregation from semantically pruned citation neighborhoods, execution-feedback iterative refinement, and subgraph level knowledge induction. Extensive experiments across three domains and ten tasks demonstrate that \method{} reduces the average performance gap to 10.04\%, improving over the strongest baseline by 24.68\%, with consistent gains across all domains.

\section{GenAI Disclosure}
Large language models (LLMs) are integral components of the proposed framework \method. Specifically, LLMs serve as ensemble rankers in SSGP for semantic relevance scoring, as relation analyzers in node-level aggregation for identifying reuse/adapt/new relationships, and as the reproduction agent for code generation and iterative refinement. In addition, generative AI tools were used to assist with language polishing of this manuscript. All scientific content, experimental design, and conclusions are solely the work of the authors.

\bibliographystyle{ACM-Reference-Format}
\bibliography{sample-base}

\newpage
\appendix
\onecolumn

\section{Impact Statement}
\label{app:impact}
This work advances automated paper reproduction by leveraging the relational structure inherent in scientific knowledge graphs. By enabling more accurate and efficient code generation from academic papers, we anticipate positive impacts on research reproducibility, accelerating the validation of published results and lowering barriers for researchers to build upon existing work. We acknowledge that automated reproduction tools could potentially be misused; however, we believe the benefits to scientific progress substantially outweigh such risks.

\section{Theorems and Proofs}
\label{app:proof}
\begin{proposition}
\label{prop:risk_averse_pruning_cantelli_group}
Fix a target paper $v_t$. For any candidate neighbor $u$, let $R(u)\in\mathbb R$
be its (random) ensemble rank (smaller is better), with
$\mu(u):=\mathbb E[R(u)]$ and $\sigma^2(u):=\mathrm{Var}(R(u))<\infty$.
For $\lambda>0$, define $s_\lambda(u):=\mu(u)+\lambda\sigma(u)$.

\textnormal{(i) Single-neighbor control.}\;
\[
\mathbb P\!\left(R(u)\ge s_\lambda(u)\right)\le \frac{1}{1+\lambda^2},
\qquad
\mathbb P\!\left(R(u)\le s_\lambda(u)\right)\ge \frac{\lambda^2}{1+\lambda^2}.
\]
% In particular, for any $\delta\in(0,1)$, choosing
% $\lambda=\sqrt{(1-\delta)/\delta}$ yields
% $\mathbb P\!\left(R(u)\le s_\lambda(u)\right)\ge 1-\delta$.

\textnormal{(ii) Group-level control (no independence needed).}\;
For any pruned set $S=\{u_1,\dots,u_K\}$,
\[
\mathbb P\!\left(\exists u\in S:\ R(u)\ge s_\lambda(u)\right)\le \frac{K}{1+\lambda^2},
\qquad
\mathbb P\!\left(\forall u\in S:\ R(u)\le s_\lambda(u)\right)\ge 1-\frac{K}{1+\lambda^2}.
\]
% In particular, if $\lambda\ge \sqrt{\frac{K}{\delta}-1}$, then
% $\mathbb P\!\left(\forall u\in S:\ R(u)\le s_\lambda(u)\right)\ge 1-\delta$.
\end{proposition}

\begin{proof}
Cantelli's inequality gives, for any $a>0$,
$\mathbb P(R(u)-\mu(u)\ge a)\le \sigma^2(u)/(\sigma^2(u)+a^2)$.
Setting $a=\lambda\sigma(u)$ yields (i).
For (ii), apply (i) to each $u\in S$ and take a union bound over the events
$\{R(u)\ge s_\lambda(u)\}$.
\end{proof}

Proposition~\ref{prop:risk_averse_pruning_cantelli_group} provides \emph{distribution-free} guarantees for the composite score $s_\lambda(u)=\mu(u)+\lambda\sigma(u)$ used in SSGP (Section~\ref{sec:semantic_pruning}): with only finite-variance assumptions, it bounds the probability that any retained neighbor is badly mis-ranked. Under a sub-Gaussian noise model motivated by ensemble LLM ranking, Theorem~\ref{thm:subgaussian_pruning} sharpens this to an exponential bound with $\lambda=O(\sqrt{\log K})$.

\begin{assumption}[Sub-Gaussian ranking noise]
\label{assump:subgaussian}
There exists $c\ge 1$ such that for all $u\in S$ and all $t\ge 0$,
\[
\mathbb P\!\left(R(u)-\mu(u)\ge t\right)
\le \exp\!\left(-\frac{t^2}{2c^2\sigma^2(u)}\right).
\]
\end{assumption}

\begin{theorem}
\label{thm:subgaussian_pruning}
Under Assumption~\ref{assump:subgaussian}, for any pruned set $S$ with $|S|=K$ and any $\lambda>0$,
\[
\mathbb P\!\left(\forall u\in S:\ R(u)\le s_\lambda(u)\right)
\ge 1- K\exp\!\left(-\frac{\lambda^2}{2c^2}\right),
\qquad s_\lambda(u):=\mu(u)+\lambda\sigma(u).
\]
% In particular, for any $\delta\in(0,1)$, choosing $\lambda\ge c\sqrt{2\log(K/\delta)}$ gives
% $\mathbb P\!\left(\forall u\in S:\ R(u)\le s_\lambda(u)\right)\ge 1-\delta$.
\end{theorem}
\begin{proof}
For any $u\in S$ and $\lambda>0$, Assumption~\ref{assump:subgaussian} gives
\[
\mathbb P\!\left(R(u)\ge s_\lambda(u)\right)
=\mathbb P\!\left(R(u)-\mu(u)\ge \lambda\sigma(u)\right)
\le \exp\!\left(-\frac{\lambda^2}{2c^2}\right).
\]
Thus, letting $A_u:=\{R(u)\ge s_\lambda(u)\}$ and applying the union bound,
\[
\mathbb P\!\left(\exists u\in S:\ R(u)\ge s_\lambda(u)\right)
\le \sum_{u\in S}\mathbb P(A_u)
\le K\exp\!\left(-\frac{\lambda^2}{2c^2}\right),
\]
so
\[
\mathbb P\!\left(\forall u\in S:\ R(u)\le s_\lambda(u)\right)
\ge 1-K\exp\!\left(-\frac{\lambda^2}{2c^2}\right).
\]
% If $\lambda\ge c\sqrt{2\log(K/\delta)}$, then
% $K\exp(-\lambda^2/(2c^2))\le K\exp(-\log(K/\delta))=\delta$, proving the claim.
\end{proof}

\section{More Experimental Details}
\label{app:more_exp}

\subsection{Effectiveness of SSGP}
\label{app:ssgp}

\begin{figure}[t]
    \centering
    \includegraphics[width=1\linewidth]{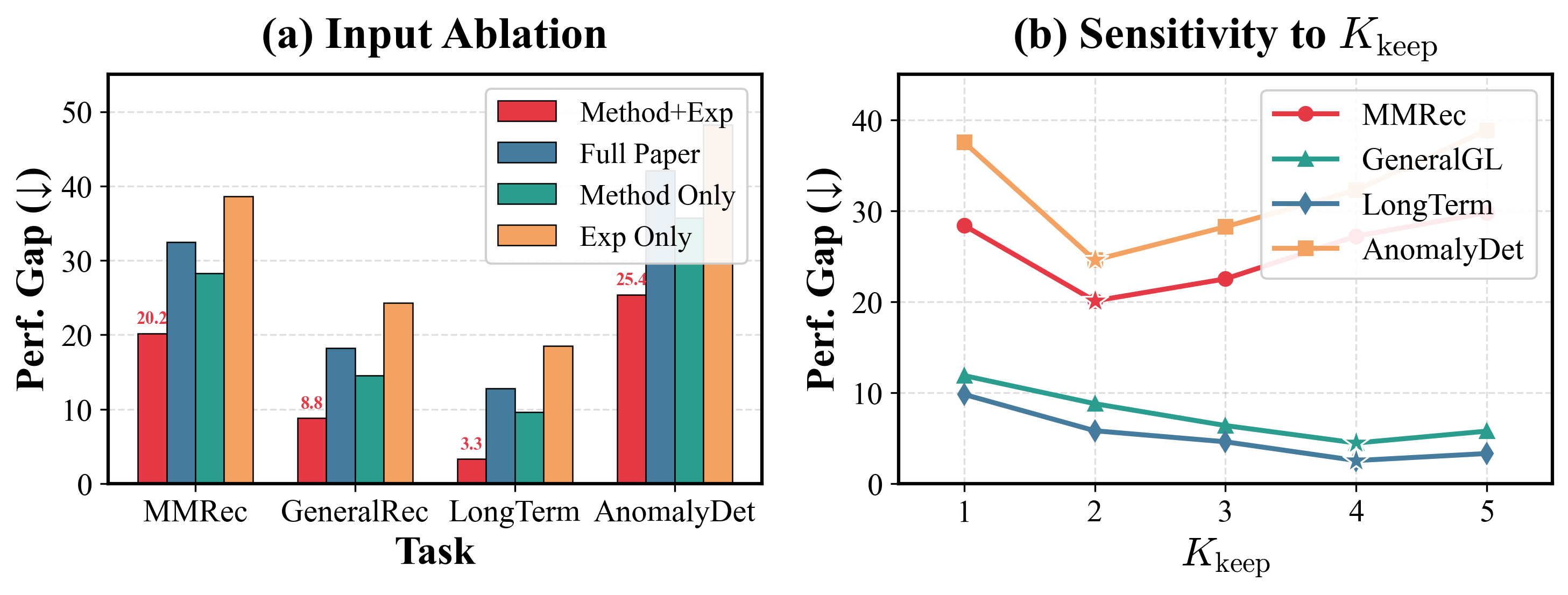}
    \caption{\textbf{SSGP analysis.} (a) Input ablation: Method+Experiments achieves optimal performance across all tasks. (b) Sensitivity to $K_{\text{keep}}$: difficult tasks require smaller $K$ for precise pruning.}
    \label{fig:ssgp_analysis}
\end{figure}

\begin{figure}[t]
  \centering
  \includegraphics[width=1\linewidth]{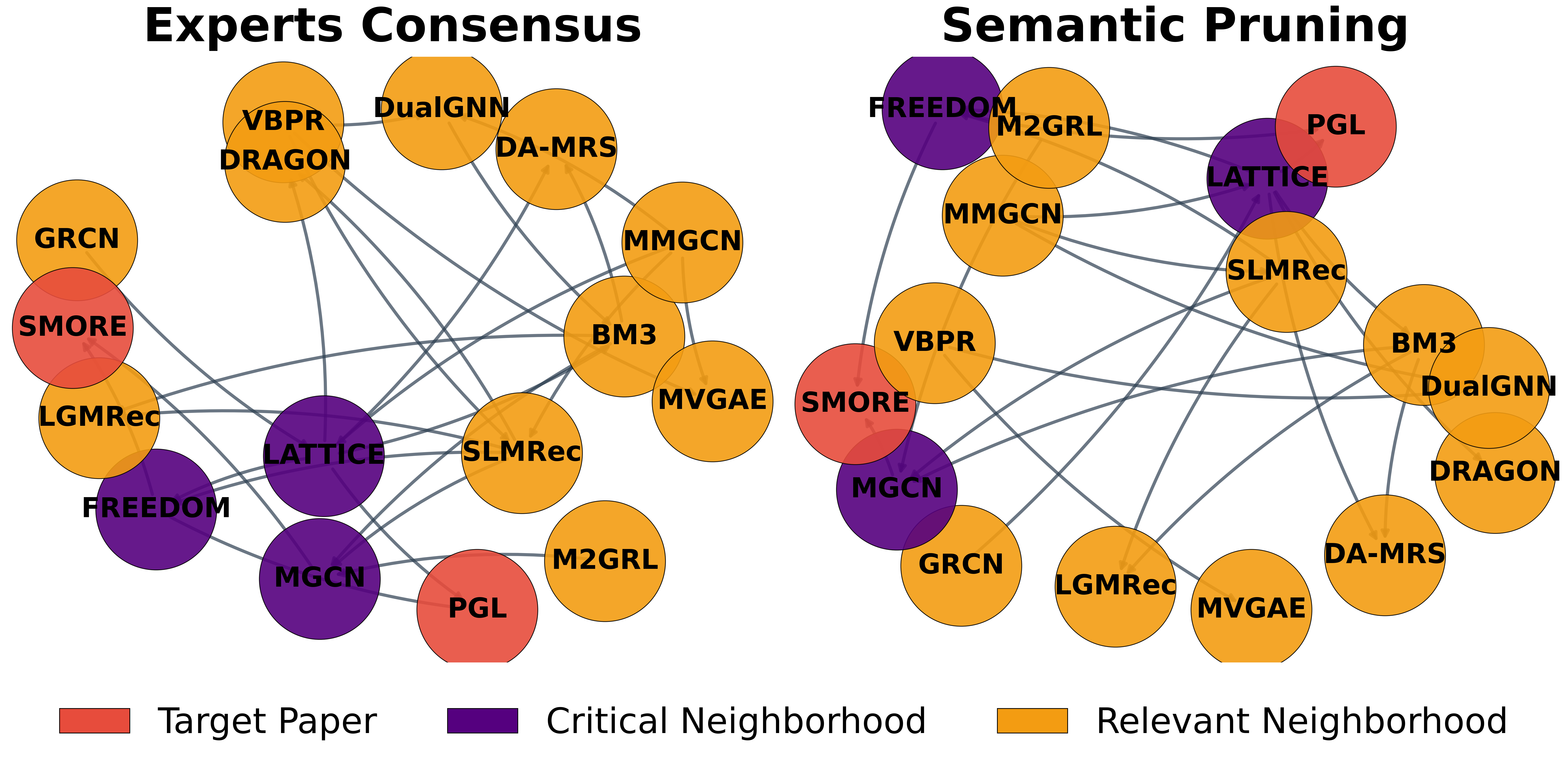}
  \caption{\textbf{Expert annotation validation.} Comparison between expert consensus (left) and our semantic pruning (right). Red: target paper; Purple: critical neighbors; Orange: relevant neighbors. Our method retains all expert-identified critical nodes.}
  \label{fig:pruning_validation}
\end{figure}

\ding{192}~\textbf{Expert Annotation Validation.} We design an expert annotation experiment by inviting six domain experts (three in recommendation systems and three in time series analysis) to independently annotate paper associations in the training sets of both domains. As illustrated in Figure~\ref{fig:pruning_validation}, our method achieves high consistency with expert consensus (Fleiss' $\kappa = 0.82$) across both domains. Notably, among neighbor papers annotated by experts as critical (purple nodes), our method completely retains these nodes. This demonstrates that semantic pruning can accurately identify knowledge sources that are truly important for reproduction tasks, effectively filtering out interference from noisy neighbors.

\ding{193}~\textbf{Input Ablation Study.} We compare four input configurations for semantic pruning: jointly inputting Method and Experiments sections, inputting the full paper, inputting only Method, and inputting only Experiments. As shown in Figure~\ref{fig:ssgp_analysis}(a), the Method+Experiments combination achieves optimal performance across all tasks. This indicates that the Method and Experiments sections contain the core implementation semantics required for reproduction, while other sections such as Introduction and Related Work may introduce noise that degrades pruning effectiveness.

\ding{194}~\textbf{Hyperparameter Sensitivity Analysis.} Figure~\ref{fig:ssgp_analysis}(b) indicates heterogeneous sensitivity to the neighbor-retention hyperparameter $K_{\text{keep}}$: tasks with larger performance gaps (MMRec, AnomalyDetection) attain their optimum at $K_{\text{keep}}=2$, reflecting the need for tighter pruning around core evidence, whereas tasks with smaller gaps (GeneralGL, LongTerm) peak at $K_{\text{keep}}=4$, showing tolerance for broader neighborhoods. This divergence underscores a trade-off between task difficulty and the degree of pruning required.

\section{Benchmark}
\label{sec:benchmark}

This section provides a detailed overview of the benchmark settings, including the specific papers and codebases used for training, validation, and testing.

\subsection{RecSys}
\label{sec:recsys}

Recommender systems are pivotal in information retrieval. We categorize the benchmark into three mainstream tasks: Multimodal Recommendation (leveraging side information), General Recommendation (classic collaborative filtering), and Sequential Recommendation (capturing dynamic user preferences).

\subsubsection{Multimodal Recommendation}
Multimodal recommendation enhances prediction accuracy by incorporating auxiliary data such as images, text, and videos. Table~\ref{tab:multimodal_papers} lists the baselines used, which utilize various fusion strategies to integrate visual and textual features with user-item interactions.

{
\small
\renewcommand{\arraystretch}{1.3}
\begin{xltabular}{\textwidth}{
    >{\raggedright\arraybackslash\hsize=1.6\hsize}X 
    >{\centering\arraybackslash\hsize=0.4\hsize}X 
}
    \caption{Selected Baselines for Multimodal Recommendation Benchmark.} \label{tab:multimodal_papers} \\
    
    \toprule
    \textbf{Paper Title} & \textbf{Date} \\ 
    \midrule
    \endfirsthead
    
    \multicolumn{2}{c}{{\bfseries \tablename\ \thetable{} -- continued from previous page}} \\
    \toprule
    \textbf{Paper Title} & \textbf{Date} \\ 
    \midrule
    \endhead
    
    \midrule
    \multicolumn{2}{r}{{Continued on next page}} \\
    \bottomrule
    \endfoot
    
    \bottomrule
    \endlastfoot

    % --- Training Set Sources ---
    \multicolumn{2}{l}{\textit{\textbf{Training Set Sources}}} \\ \midrule
    SelfCF: A Simple Framework for Self-supervised Collaborative Filtering & 2021.07 \\ 
    Layer-refined Graph Convolutional Networks for Recommendation & 2022.07 \\ 
    Visual Bayesian Personalized Ranking from Implicit Feedback & 2015.10 \\ 
    MMGCN: Multi-modal Graph Convolution Network for Personalized Recommendation of Micro-video & 2019.10 \\ 
    Are We Really Making Much Progress? A Worrying Analysis of Recent Neural Recommendation Approaches & 2019.07 \\
    Graph-Refined Convolutional Network for Multimedia Recommendation with Implicit Feedback & 2021.11 \\ 
    DualGNN: Dual Graph Neural Network for Multimedia Recommendation & 2021.12 \\ 
    Mining Latent Structures for Multimedia Recommendation & 2021.04 \\ 
    Self-Supervised Learning for Multimedia Recommendation & 2022.06 \\ 
    Bootstrap Latent Representations for Multi-modal Recommendation & 2023.04 \\ 
    A Tale of Two Graphs: Freezing and Denoising Graph Structures for Multimodal Recommendation & 2022.11 \\ 
    Multi-View Graph Convolutional Network for Multimedia Recommendation & 2023.08 \\ 
    Enhancing Dyadic Relations with Homogeneous Graphs for Multimodal Recommendation & 2023.01 \\ 

    % --- Validation Set Sources ---
    \midrule
    \multicolumn{2}{l}{\textit{\textbf{Validation Set Sources}}} \\ \midrule
    LGMRec: Local and Global Graph Learning for Multimodal Recommendation & 2023.12 \\ 
    Improving Multi-modal Recommender Systems by Denoising and Aligning Multi-modal Content and User Feedback & 2024.06 \\ 

    % --- Test Set Source ---
    \midrule
    \multicolumn{2}{l}{\textit{\textbf{Test Set Sources}}} \\ \midrule
    Harnessing Multimodal Large Language Models for Multimodal Sequential Recommendation & 2025.01 \\ 
    NoteLLM-2: Multimodal Large Representation Models for Recommendation & 2025.01 \\ 
    LEMUR: Large scale End-to-end MUltimodal Recommendation & 2025.11 \\ 
    Multimodal Recommendation System Based on Cross Self-Attention Fusion & 2025.01 \\ 
\end{xltabular}
}

\subsubsection{General Recommendation}
General recommendation, often referred to as Collaborative Filtering (CF), focuses on learning user and item representations solely from interaction history. Table~\ref{tab:cf_papers} summarizes the models evaluated, spanning from matrix factorization to modern graph contrastive learning approaches.

{
\small
\renewcommand{\arraystretch}{1.3}
\begin{xltabular}{\textwidth}{
    >{\raggedright\arraybackslash\hsize=1.6\hsize}X 
    >{\centering\arraybackslash\hsize=0.4\hsize}X 
}
    \caption{Selected Baselines for General Recommendation (Collaborative Filtering).} \label{tab:cf_papers} \\
    
    \toprule
    \textbf{Paper Title} & \textbf{Date} \\ 
    \midrule
    \endfirsthead
    
    \multicolumn{2}{c}{{\bfseries \tablename\ \thetable{} -- continued from previous page}} \\
    \toprule
    \textbf{Paper Title} & \textbf{Date} \\ 
    \midrule
    \endhead
    
    \midrule
    \multicolumn{2}{r}{{Continued on next page}} \\
    \bottomrule
    \endfoot
    
    \bottomrule
    \endlastfoot

    % --- Training Set Sources ---
    \multicolumn{2}{l}{\textit{\textbf{Training Set Sources}}} \\ \midrule
    ItemKNN: Item-based Collaborative Filtering Recommendation Algorithms & 2024.07 \\ 
    BPR: Bayesian Personalized Ranking from Implicit Feedback & 2012.05 \\ 
    ENMF: Efficient Neural Matrix Factorization without Sampling & 2020.01 \\ 
    AsymKNN: Factorization Meets the Neighborhood & 2008.08 \\ 
    SLIM: Sparse Linear Methods for Top-N Recommender Systems & 2011.12 \\ 
    FISM: Factored Item Similarity Models for Top-N Recommender Systems & 2013.08 \\ 
    CDAE: Collaborative Denoising Auto-Encoders for Top-N RecSys & 2016.02 \\ 
    LINE: Large-scale Information Network Embedding & 2015.03 \\ 
    DMF: Deep Matrix Factorization Models for Recommender Systems & 2017 \\ 
    NeuMF: Neural Collaborative Filtering & 2017.08 \\ 
    Variational Autoencoders for Collaborative Filtering & 2018.02 \\ 
    NAIS: Neural Attentive Item Similarity Model for Recommendation & 2018.09 \\ 
    GCMC: Graph Convolutional Matrix Completion & 2017.06 \\ 
    ConvNCF: Outer Product-based Neural Collaborative Filtering & 2018.08 \\ 
    SpectralCF: Spectral Collaborative Filtering & 2018.08 \\ 
    MacridVAE: Learning Disentangled Representations for Recommendation & 2019.10 \\ 
    RecVAE: A New Variational Autoencoder with Implicit Feedback & 2019.12 \\ 
    NGCF: Neural Graph Collaborative Filtering & 2019.05 \\ 
    EASE: Embarrassingly Shallow Autoencoders for Sparse Data & 2019.05 \\ 
    LightGCN: Simplifying and Powering Graph Convolution Network & 2020.02 \\ 
    DGCF: Disentangled Graph Collaborative Filtering & 2020.07 \\ 
    NCL: Improving Graph CF with Neighborhood-enriched Contrastive Learning & 2022.02 \\ 
    NNCF: A Neural Collaborative Filtering Model with Interaction-based Neighborhood & 2017.11 \\ 
    RaCT: Towards Amortized Ranking-Critical Training for CF & 2019.06 \\ 
    SimpleX: A Simple and Strong Baseline for Collaborative Filtering & 2021.09 \\ 

    % --- Validation Set Sources ---
    \midrule
    \multicolumn{2}{l}{\textit{\textbf{Validation Set Sources}}} \\ \midrule
    SGL: Self-supervised Graph Learning for Recommendation & 2020.10 \\ 
    NCE-PL: Noise-Contrastive Estimation with Pseudo-Labeling & 2023.05 \\ 
    LDiffRec: Simplifying and Powering Diffusion Model for Multimedia Rec & 2023.09 \\ 

    % --- Test Set Sources ---
    \midrule
    \multicolumn{2}{l}{\textit{\textbf{Test Set Sources}}} \\ \midrule
    NLGCL: Naturally Existing Neighbor Layers Graph Contrastive Learning & 2025.09 \\ 
    FlowCF: Flow-Based Contrastive Learning for Collaborative Filtering & 2025.09 \\ 
    CORONA: A Coarse-to-Fine Framework for Graph-based Recommendation with LLMs & 2025.06 \\ 
    EAGER-LLM: Enhancing LLMs as Recommenders through Behavior-Semantic Integration & 2025.02 \\ 
    Collaborative Diffusion Model for Recommender System & 2025.01 \\ 
    CUGF: A Reliable and Fair Recommendation Framework & 2025.04 \\
\end{xltabular}
}

\subsubsection{Sequence Recommendation}
Sequential recommendation aims to model the dynamic evolution of user preferences over time. Table~\ref{tab:seq_rec_papers} outlines the comprehensive list of baselines, ranging from Markov Chains and RNNs to recent Transformer-based and Diffusion-based architectures.

{
\small
\renewcommand{\arraystretch}{1.3}
\begin{xltabular}{\textwidth}{
    >{\raggedright\arraybackslash\hsize=1.6\hsize}X 
    >{\centering\arraybackslash\hsize=0.4\hsize}X 
}
    \caption{Selected Baselines for Sequential Recommendation Benchmark.} \label{tab:seq_rec_papers} \\
    
    \toprule
    \textbf{Paper Title} & \textbf{Date} \\ 
    \midrule
    \endfirsthead
    
    \multicolumn{2}{c}{{\bfseries \tablename\ \thetable{} -- continued from previous page}} \\
    \toprule
    \textbf{Paper Title} & \textbf{Date} \\ 
    \midrule
    \endhead
    
    \midrule
    \multicolumn{2}{r}{{Continued on next page}} \\
    \bottomrule
    \endfoot
    
    \bottomrule
    \endlastfoot

    % --- Training Set Sources ---
    \multicolumn{2}{l}{\textit{\textbf{Training Set Sources}}} \\ \midrule
    FPMC: Factorizing Personalized Markov Chains for Next-Basket Rec & 2010.04 \\ 
    GRU4Rec: Session-based Recommendations with Recurrent Neural Networks & 2015.11 \\ 
    TransRec: Translation-based Recommendation & 2017.07 \\ 
    DIN: Deep Interest Network for Click-Through Rate Prediction & 2017.06 \\ 
    STAMP: Short-Term Attention/Memory Priority Model & 2018.07 \\ 
    Caser: Personalized Top-N Sequential Rec via Convolutional Embedding & 2018.09 \\ 
    NARM: Neural Attentive Session-based Recommendation & 2017.11 \\ 
    NextItNet: A Simple Convolutional Generative Network & 2018.08 \\ 
    FOSSIL: Fusing Similarity Models with Markov Chains & 2016.09 \\ 
    SASRec: Self-Attentive Sequential Recommendation & 2018.08 \\ 
    DIEN: Deep Interest Evolution Network for CTR Prediction & 2018.09 \\ 
    RepeatNet: A Repeat Aware Neural Recommendation Machine & 2018.12 \\ 
    SRGNN: Session-based Recommendation with Graph Neural Networks & 2018.11 \\ 
    GRU4RecF: Parallel RNN Architectures for Feature-rich Recommendations & 2016.09 \\ 
    Latent Cross: Making Use of Context in Recurrent Recommender Systems & 2018.02 \\ 
    HGN: Hierarchical Gating Networks for Sequential Recommendation & 2019.06 \\ 
    BERT4Rec: Sequential Recommendation with BERT & 2019.04 \\ 
    FDSA: Feature-level Deeper Self-Attention Network & 2019.08 \\ 
    HRM: Learning Hierarchical Representation Model for Next Basket Rec & 2015.08 \\ 
    KSR: Improving Sequential Recommendation with Knowledge-Enhanced Memory & 2018.06 \\ 
    Lightweight Self-Attentive Sequential Recommendation & 2021.08 \\ 
    CORE: Simple and Effective Session-based Recommendation & 2022.04 \\ 
    S3-Rec: Self-Supervised Learning with Mutual Information Maximization & 2020.08 \\ 
    Time Lag Aware Sequential Recommendation & 2022.08 \\ 
    GCSAN: Graph Contextualized Self-Attention Network & 2019.08 \\ 
    NPE: Neural Personalized Embedding for Collaborative Filtering & 2018.05 \\ 
    DiffuRec: A Diffusion Model for Sequential Recommendation & 2023.10 \\ 

    % --- Validation Set Sources ---
    \midrule
    \multicolumn{2}{l}{\textit{\textbf{Validation Set Sources}}} \\ \midrule
    SINE: Sparse-Interest Network for Sequential Recommendation & 2021.02 \\ 
    FEARec: Frequency Enhanced Hybrid Attention Network & 2023.04 \\ 
    BASRec: Augmenting Sequential Rec with Balanced Relevance & 2024.12 \\ 

    % --- Test Set Sources ---
    \midrule
    \multicolumn{2}{l}{\textit{\textbf{Test Set Sources}}} \\ \midrule
    Not Just What, But When: Integrating Irregular Intervals to LLM & 2025.07 \\ 
    On-Device Large Language Models for Sequential Recommendation & 2026.01 \\ 
    BlossomRec: Block-level Fused Sparse Attention Mechanism for Sequential Recommendations & 2025.12 \\ 
    Dual Conditional Diffusion Models for Sequential Recommendation & 2025.03 \\ 
    Improving Sequential Recommenders through Counterfactual Augmentation of System Exposure & 2025.04 \\ 
\end{xltabular}
}

\subsection{Time Series Analysis Tasks}

\subsubsection{Anomaly Detection}
To rigorously evaluate the performance of anomaly detection models, we selected a diverse set of baselines ranging from foundational Transformer architectures to recent frequency-domain innovations. Table~\ref{tab:anomaly_papers} summarizes the sources and publication dates of these methods.

{
\small
\renewcommand{\arraystretch}{1.3}
\begin{xltabular}{\textwidth}{
    >{\raggedright\arraybackslash\hsize=0.8\hsize}X
    >{\raggedright\arraybackslash\hsize=1.1\hsize}X
    >{\raggedright\arraybackslash\hsize=1.1\hsize}X 
    c
}
    \caption{Summary of Baselines and SOTA Models for Time Series Anomaly Detection.} \label{tab:anomaly_papers} \\
    
    \toprule
    \textbf{Paper Title} & \textbf{arXiv} & \textbf{GitHub} & \textbf{Date} \\ 
    \midrule
    \endfirsthead
    
    \multicolumn{4}{c}%
    {{\bfseries \tablename\ \thetable{} -- continued from previous page}} \\
    \toprule
    \textbf{Paper Title} & \textbf{arXiv} & \textbf{GitHub} & \textbf{Date} \\ 
    \midrule
    \endhead
    
    \midrule
    \multicolumn{4}{r}{{Continued on next page}} \\
    \bottomrule
    \endfoot
    
    \bottomrule
    \endlastfoot

    % --- Training Set Sources ---
    \multicolumn{4}{l}{\textit{\textbf{Training Set Sources}}} \\ \midrule
    Attention Is All You Need & \url{https://arxiv.org/pdf/1706.03762} & \multicolumn{1}{c}{--} & 2017.06 \\ 
    Reformer: The Efficient Transformer & \url{https://arxiv.org/pdf/2001.04451} & \url{https://github.com/google/trax} & 2020.01 \\ 
    Informer: Beyond Efficient Transformer & \url{https://arxiv.org/abs/2012.07436} & \url{https://github.com/zhouhaoyi/Informer2020} & 2020.12 \\ 
    Autoformer: Decomposition Transformers & \url{https://arxiv.org/pdf/2106.13008} & \url{https://github.com/thuml/Autoformer} & 2021.06 \\ 
    Pyraformer: Low-Complexity Pyramidal Attention & \url{https://openreview.net/pdf?id=0EXmFzUn5I} & \url{https://github.com/ant-research/Pyraformer} & 2022.01 \\ 
    FEDformer: Frequency Enhanced Transformer & \url{https://arxiv.org/pdf/2201.12740} & \url{https://github.com/MAZiqing/FEDformer} & 2022.01 \\ 
    ETSformer: Exponential Smoothing Transformers & \url{https://arxiv.org/pdf/2202.01381} & \url{https://github.com/salesforce/ETSformer} & 2022.02 \\ 
    Are Transformers Effective for Forecasting? & \url{https://arxiv.org/pdf/2205.13504} & \url{https://github.com/cure-lab/LTSF-Linear} & 2022.05 \\ 
    FiLM: Frequency improved Legendre Memory & \url{https://arxiv.org/pdf/2205.08897} & \url{https://github.com/tianzhou2011/FiLM/} & 2022.05 \\ 
    LightTS: Lightweight Time Series Classification & \url{https://arxiv.org/pdf/2302.12721} & \url{https://github.com/d-gcc/LightTS} & 2023.02 \\ 
    TimesNet: Temporal 2D-Variation Modeling & \url{https://arxiv.org/pdf/2210.02186} & \url{https://github.com/thuml/TimesNet} & 2022.10 \\ 
    Crossformer: Transformer Utilizing Cross-Dim & \url{https://openreview.net/pdf?id=vSVLM2j9eie} & \url{https://github.com/Thinklab-SJTU/Crossformer} & 2023.02 \\ 
    MICN: Multi-scale Local and Global Context & \url{https://openreview.net/pdf?id=zt53IDUR1U} & \url{https://github.com/wanghq21/MICN} & 2023.02 \\ 
    iTransformer: Inverted Transformers & \url{https://arxiv.org/pdf/2310.06625} & \url{https://github.com/thuml/iTransformer} & 2023.10 \\ 
    
    % --- Validation Set Source ---
    \midrule
    \multicolumn{4}{l}{\textit{\textbf{Validation Set Sources}}} \\ \midrule
    KAN-AD: Kolmogorov-Arnold Networks & \url{https://arxiv.org/pdf/2411.00278} & \url{https://github.com/CSTCloudOps/KAN-AD} & 2024.11 \\
    
    % --- Test Set Source ---
    \midrule
    \multicolumn{4}{l}{\textit{\textbf{Test Set Sources}}} \\ \midrule
    TIMEMIXER++: Universal Predictive Analysis & \url{https://openreview.net/pdf?id=1CLzLXSFNn} & \url{https://github.com/kwuking/TimeMixer} & 2025.05 \\ 
    MtsCID: Capturing Coarse-Grained Dependencies & \url{https://arxiv.org/pdf/2501.16364} & \url{https://github.com/ilwoof/MtsCID/} & 2025.01 \\ 
    CATCH: Channel-Aware Multivariate Time Series Anomaly Detection via Frequency Patching & \url{https://openreview.net/pdf?id=m08aK3xxdJ} & \url{https://github.com/decisionintelligence/CATCH} & 2025.01 \\ 
    CrossAD: Time Series Anomaly Detection with Cross-scale Associations and Cross-window Modeling & \url{https://arxiv.org/pdf/2510.12489} & \url{https://github.com/decisionintelligence/CrossAD} & 2025.10 \\
\end{xltabular}
}

\subsubsection{Classification}
For the classification task, we have curated a list of models that represent significant milestones in the field. Table~\ref{tab:classification_papers} details the configuration of our classification benchmark, distinguishing between established training baselines and the validation/test sets derived from the latest research.

{
\small
\renewcommand{\arraystretch}{1.3}
\begin{xltabular}{\textwidth}{
    >{\raggedright\arraybackslash\hsize=0.8\hsize}X 
    >{\raggedright\arraybackslash\hsize=1.1\hsize}X 
    >{\raggedright\arraybackslash\hsize=1.1\hsize}X 
    c
}
    \caption{List of Selected Models for Time Series Classification Benchmark.} \label{tab:classification_papers} \\
    
    \toprule
    \textbf{Paper Title} & \textbf{arXiv} & \textbf{GitHub} & \textbf{Date} \\ 
    \midrule
    \endfirsthead
    
    \multicolumn{4}{c}%
    {{\bfseries \tablename\ \thetable{} -- continued from previous page}} \\
    \toprule
    \textbf{Paper Title} & \textbf{arXiv} & \textbf{GitHub} & \textbf{Date} \\ 
    \midrule
    \endhead
    
    \midrule
    \multicolumn{4}{r}{{Continued on next page}} \\
    \bottomrule
    \endfoot
    
    \bottomrule
    \endlastfoot

    % --- Training Set Sources ---
    \multicolumn{4}{l}{\textit{\textbf{Training Set Sources}}} \\ \midrule
    Attention Is All You Need & \url{https://arxiv.org/pdf/1706.03762} & \multicolumn{1}{c}{--} & 2017.06 \\ 
    Reformer: The Efficient Transformer & \url{https://arxiv.org/pdf/2001.04451} & \url{https://github.com/google/trax} & 2020.01 \\ 
    Informer: Beyond Efficient Transformer & \url{https://arxiv.org/pdf/2012.07436} & \url{https://github.com/zhouhaoyi/Informer2020} & 2020.12 \\ 
    Autoformer: Decomposition Transformers & \url{https://arxiv.org/pdf/2106.13008} & \url{https://github.com/thuml/Autoformer} & 2021.06 \\ 
    Pyraformer: Low-Complexity Pyramidal Attention & \url{https://openreview.net/pdf?id=0EXmFzUn5I} & \url{https://github.com/ant-research/Pyraformer} & 2022.01 \\ 
    FEDformer: Frequency Enhanced Transformer & \url{https://arxiv.org/pdf/2201.12740} & \url{https://github.com/MAZiqing/FEDformer} & 2022.01 \\ 
    ETSformer: Exponential Smoothing Transformers & \url{https://arxiv.org/pdf/2202.01381} & \url{https://github.com/salesforce/ETSformer} & 2022.02 \\ 
    Are Transformers Effective for Forecasting? & \url{https://arxiv.org/pdf/2205.13504} & \url{https://github.com/cure-lab/LTSF-Linear} & 2022.05 \\ 
    FiLM: Frequency improved Legendre Memory & \url{https://arxiv.org/pdf/2205.08897} & \url{https://github.com/tianzhou2011/FiLM/} & 2022.05 \\ 
    LightTS: Lightweight Time Series Classification & \url{https://arxiv.org/pdf/2302.12721} & \url{https://github.com/d-gcc/LightTS} & 2023.02 \\ 
    TimesNet: Temporal 2D-Variation Modeling & \url{https://arxiv.org/pdf/2210.02186} & \url{https://github.com/thuml/TimesNet} & 2022.10 \\ 
    PatchTST: A Time Series is Worth 64 Words & \url{https://arxiv.org/pdf/2211.14730} & \multicolumn{1}{c}{--} & 2022.11 \\ 
    Crossformer: Utilizing Cross-Dim Dependency & \url{https://openreview.net/pdf?id=vSVLM2j9eie} & \url{https://github.com/Thinklab-SJTU/Crossformer} & 2023.02 \\ 
    MICN: Multi-scale Local and Global Context & \url{https://openreview.net/pdf?id=zt53IDUR1U} & \url{https://github.com/wanghq21/MICN} & 2023.02 \\ 

    % --- Validation Set Source ---
    \midrule
    \multicolumn{4}{l}{\textit{\textbf{Validation Set Sources}}} \\ \midrule
    iTransformer: Inverted Transformers & \url{https://arxiv.org/pdf/2310.06625} & \url{https://github.com/thuml/iTransformer} & 2023.10 \\ 
    TimeMixer: Decomposable Multiscale Mixing & \url{https://arxiv.org/pdf/2405.14616} & \multicolumn{1}{c}{--} & 2024.05 \\ 

    % --- Test Set Source ---
    \midrule
    \multicolumn{4}{l}{\textit{\textbf{Test Set Sources}}} \\ \midrule
    TIMEMIXER++: Universal Predictive Analysis & \url{https://openreview.net/pdf?id=1CLzLXSFNn} & \url{https://github.com/kwuking/TimeMixer} & 2025.05 \\ 
    InterpretGatedNetwork: Shedding Light on TSC & \url{https://openreview.net/pdf?id=n34taxF0TC} & \url{https://github.com/YunshiWen/Ign} & 2025.01 \\ 
    SGN: Shifted Window-Based Variable Grouping & \url{https://neurips.cc/virtual/2025/poster/115248} & \url{https://github.com/colison/SGN} & 2025.12 \\ 
\end{xltabular}
}

\subsubsection{Imputation}
Imputation is a critical preprocessing step for real-world time series analysis. Table~\ref{tab:imputation_papers} lists the baselines evaluated for the imputation task, ranging from reconstruction-based methods to recent generative approaches.

{
\small
\renewcommand{\arraystretch}{1.3}
\begin{xltabular}{\textwidth}{
    >{\raggedright\arraybackslash\hsize=0.8\hsize}X 
    >{\raggedright\arraybackslash\hsize=1.1\hsize}X 
    >{\raggedright\arraybackslash\hsize=1.1\hsize}X 
    c
}
    \caption{Summary of Baselines and SOTA Models for Time Series Imputation.} \label{tab:imputation_papers} \\
    
    \toprule
    \textbf{Paper Title} & \textbf{arXiv} & \textbf{GitHub} & \textbf{Date} \\ 
    \midrule
    \endfirsthead
    
    \multicolumn{4}{c}{{\bfseries \tablename\ \thetable{} -- continued from previous page}} \\
    \toprule
    \textbf{Paper Title} & \textbf{arXiv} & \textbf{GitHub} & \textbf{Date} \\ 
    \midrule
    \endhead
    
    \midrule
    \multicolumn{4}{r}{{Continued on next page}} \\
    \bottomrule
    \endfoot
    
    \bottomrule
    \endlastfoot

    % --- Training Set Sources ---
    \multicolumn{4}{l}{\textit{\textbf{Training Set Sources}}} \\ \midrule
    Attention Is All You Need & \url{https://arxiv.org/pdf/1706.03762} & \multicolumn{1}{c}{--} & 2017.06 \\ 
    Reformer: The Efficient Transformer & \url{https://arxiv.org/pdf/2001.04451} & \url{https://github.com/google/trax} & 2020.01 \\ 
    Informer: Beyond Efficient Transformer & \url{https://arxiv.org/pdf/2012.07436} & \url{https://github.com/zhouhaoyi/Informer2020} & 2020.12 \\ 
    Autoformer: Decomposition Transformers & \url{https://arxiv.org/pdf/2106.13008} & \url{https://github.com/thuml/Autoformer} & 2021.06 \\ 
    Pyraformer: Low-Complexity Pyramidal Attention & \url{https://openreview.net/pdf?id=0EXmFzUn5I} & \url{https://github.com/ant-research/Pyraformer} & 2022.01 \\ 
    FEDformer: Frequency Enhanced Transformer & \url{https://arxiv.org/pdf/2201.12740} & \url{https://github.com/MAZiqing/FEDformer} & 2022.01 \\ 
    ETSformer: Exponential Smoothing Transformers & \url{https://arxiv.org/pdf/2202.01381} & \url{https://github.com/salesforce/ETSformer} & 2022.02 \\ 
    Are Transformers Effective for Forecasting? & \url{https://arxiv.org/pdf/2205.13504} & \url{https://github.com/cure-lab/LTSF-Linear} & 2022.05 \\ 
    Non-stationary Transformers & \url{https://arxiv.org/pdf/2205.14415} & \url{https://github.com/thuml/Nonstationary_Transformers} & 2022.05 \\ 
    FiLM: Frequency improved Legendre Memory & \url{https://arxiv.org/pdf/2205.08897} & \url{https://github.com/tianzhou2011/FiLM/} & 2022.05 \\ 
    LightTS: Lightweight Time Series Classification & \url{https://arxiv.org/pdf/2302.12721} & \url{https://github.com/d-gcc/LightTS} & 2023.02 \\ 
    TimesNet: Temporal 2D-Variation Modeling & \url{https://arxiv.org/pdf/2210.02186} & \url{https://github.com/thuml/TimesNet} & 2022.10 \\ 
    Crossformer: Utilizing Cross-Dim Dependency & \url{https://openreview.net/pdf?id=vSVLM2j9eie} & \url{https://github.com/Thinklab-SJTU/Crossformer} & 2023.02 \\ 
    MICN: Multi-scale Local and Global Context & \url{https://openreview.net/pdf?id=zt53IDUR1U} & \url{https://github.com/wanghq21/MICN} & 2023.02 \\ 
    TiDE: Time-series Dense Encoder & \url{https://arxiv.org/pdf/2304.08424} & \multicolumn{1}{c}{--} & 2023.04 \\ 

    % --- Validation Set Source ---
    \midrule
    \multicolumn{4}{l}{\textit{\textbf{Validation Set Sources}}} \\ \midrule
    iTransformer: Inverted Transformers & \url{https://arxiv.org/pdf/2310.06625} & \url{https://github.com/thuml/iTransformer} & 2023.10 \\ 
    TimeMixer: Decomposable Multiscale Mixing & \url{https://arxiv.org/pdf/2405.14616} & \url{https://github.com/kwuking/TimeMixer} & 2024.05 \\ 

% --- Test Set Source ---
    \midrule
    \multicolumn{4}{l}{\textit{\textbf{Test Set Sources}}} \\ \midrule
    Optimal Transport for Time-Series Imputation & \url{https://openreview.net/pdf?id=xPTzjpIQNp} & \url{https://github.com/FMLYD/PSW-I} & 2025.01 \\ 
    TimeFilter: Patch-Specific Spatial-Temporal Graph Filtration for Time Series Forecasting & \url{https://arxiv.org/pdf/2501.13041} & \url{https://github.com/thuml/Time-Series-Library/blob/main/models/TimeFilter.py} & 2025.01 \\ 
    ImputeINR: Time Series Imputation via Implicit Neural Representations for Disease Diagnosis with Missing Data & \url{https://arxiv.org/pdf/2505.10856} & \url{https://github.com/Leanna97/ImputeINR} & 2025.05 \\ 
\end{xltabular}
}

\subsubsection{Long-term Forecast}
Long-term forecasting evaluates the model's ability to capture long-range dependencies. Table~\ref{tab:longterm_papers} provides a comprehensive configuration of the models used, including established MLP-based and Transformer-based architectures, as well as the most recent state-of-the-art methods.

{
\small
\renewcommand{\arraystretch}{1.3}
\begin{xltabular}{\textwidth}{
    >{\raggedright\arraybackslash\hsize=0.8\hsize}X 
    >{\raggedright\arraybackslash\hsize=1.1\hsize}X 
    >{\raggedright\arraybackslash\hsize=1.1\hsize}X 
    c
}
    \caption{Detailed Configuration of Models for Long-term Time Series Forecasting.} \label{tab:longterm_papers} \\
    
    \toprule
    \textbf{Paper Title} & \textbf{arXiv} & \textbf{GitHub} & \textbf{Date} \\ 
    \midrule
    \endfirsthead
    
    \multicolumn{4}{c}{{\bfseries \tablename\ \thetable{} -- continued from previous page}} \\
    \toprule
    \textbf{Paper Title} & \textbf{arXiv} & \textbf{GitHub} & \textbf{Date} \\ 
    \midrule
    \endhead
    
    \midrule
    \multicolumn{4}{r}{{Continued on next page}} \\
    \bottomrule
    \endfoot
    
    \bottomrule
    \endlastfoot

    % --- Training Set Sources ---
    \multicolumn{4}{l}{\textit{\textbf{Training Set Sources}}} \\ \midrule
    Attention Is All You Need & \url{https://arxiv.org/pdf/1706.03762} & \multicolumn{1}{c}{--} & 2017.06 \\ 
    Reformer: The Efficient Transformer & \url{https://arxiv.org/pdf/2001.04451} & \url{https://github.com/google/trax} & 2020.01 \\ 
    Informer: Beyond Efficient Transformer & \url{https://arxiv.org/pdf/2012.07436} & \url{https://github.com/zhouhaoyi/Informer2020} & 2020.12 \\ 
    Autoformer: Decomposition Transformers & \url{https://arxiv.org/pdf/2106.13008} & \url{https://github.com/thuml/Autoformer} & 2021.06 \\ 
    Pyraformer: Low-Complexity Pyramidal Attention & \url{https://openreview.net/pdf?id=0EXmFzUn5I} & \url{https://github.com/ant-research/Pyraformer} & 2022.01 \\ 
    FEDformer: Frequency Enhanced Transformer & \url{https://arxiv.org/pdf/2201.12740} & \url{https://github.com/MAZiqing/FEDformer} & 2022.01 \\ 
    ETSformer: Exponential Smoothing Transformers & \url{https://arxiv.org/pdf/2202.01381} & \url{https://github.com/salesforce/ETSformer} & 2022.02 \\ 
    Are Transformers Effective for Forecasting? & \url{https://arxiv.org/pdf/2205.13504} & \url{https://github.com/cure-lab/LTSF-Linear} & 2022.05 \\ 
    Non-stationary Transformers & \url{https://arxiv.org/pdf/2205.14415} & \url{https://github.com/thuml/Nonstationary_Transformers} & 2022.05 \\ 
    FiLM: Frequency improved Legendre Memory & \url{https://arxiv.org/pdf/2205.08897} & \url{https://github.com/tianzhou2011/FiLM/} & 2022.05 \\ 
    LightTS: Lightweight Time Series Classification & \url{https://arxiv.org/pdf/2302.12721} & \url{https://github.com/d-gcc/LightTS} & 2023.02 \\ 
    TimesNet: Temporal 2D-Variation Modeling & \url{https://arxiv.org/pdf/2210.02186} & \url{https://github.com/thuml/TimesNet} & 2022.10 \\ 
    PatchTST: A Time Series is Worth 64 Words & \url{https://arxiv.org/pdf/2211.14730} & \multicolumn{1}{c}{--} & 2022.11 \\ 
    Crossformer: Utilizing Cross-Dim Dependency & \url{https://openreview.net/pdf?id=vSVLM2j9eie} & \url{https://github.com/Thinklab-SJTU/Crossformer} & 2023.02 \\ 
    MICN: Multi-scale Local and Global Context & \url{https://openreview.net/pdf?id=zt53IDUR1U} & \url{https://github.com/wanghq21/MICN} & 2023.02 \\ 
    TSMixer: An All-MLP Architecture & \url{https://arxiv.org/pdf/2303.06053} & \url{https://github.com/google-research/tsmixer} & 2023.03 \\ 
    TiDE: Time-series Dense Encoder & \url{https://arxiv.org/pdf/2304.08424} & \multicolumn{1}{c}{--} & 2023.04 \\ 
    Koopa: Learning Non-stationary Dynamics & \url{https://arxiv.org/pdf/2305.18803} & \url{https://github.com/thuml/Koopa} & 2023.05 \\ 
    SegRNN: Segment Recurrent Neural Network & \url{https://arxiv.org/pdf/2308.11200} & \multicolumn{1}{c}{--} & 2023.08 \\ 
    iTransformer: Inverted Transformers & \url{https://arxiv.org/pdf/2310.06625} & \url{https://github.com/thuml/iTransformer} & 2023.10 \\ 
    TimeMixer: Decomposable Multiscale Mixing & \url{https://arxiv.org/pdf/2405.14616} & \url{https://github.com/kwuking/TimeMixer} & 2024.05 \\ 
    TimeXer: Empowering Transformers with Exogenous & \url{https://arxiv.org/pdf/2402.19072} & \url{https://github.com/thuml/TimeXer} & 2024.02 \\ 
    WPMixer: Efficient Multi-Resolution Mixing & \url{https://arxiv.org/pdf/2412.17176} & \url{https://github.com/Secure-and-Intelligent-Systems-Lab/WPMixer} & 2024.12 \\ 

    % --- Validation Set Source ---
    \midrule
    \multicolumn{4}{l}{\textit{\textbf{Validation Set Sources}}} \\ \midrule
    Sentinel: Multi-Patch Transformer & \url{https://arxiv.org/pdf/2503.17658} & \multicolumn{1}{c}{--} & 2025.03 \\ 
    Mamba: Linear-Time Sequence Modeling & \url{https://arxiv.org/pdf/2312.00752} & \url{https://github.com/state-spaces/mamba} & 2023.12 \\ 

    % --- Test Set Source ---
    \midrule
    \multicolumn{4}{l}{\textit{\textbf{Test Set Sources}}} \\ \midrule
    Unlocking the Power of LSTM for Long Term & \url{https://ojs.aaai.org/index.php/AAAI/article/view/33303} & \url{https://github.com/Eleanorkong/P-sLSTM} & 2025.01 \\ 
    TimeEmb: Lightweight Static-Dynamic & \url{https://arxiv.org/abs/2510.00461} & \url{https://github.com/showmeon/TimeEmb} & 2025.10 \\ 
    Efficiently Enhancing Long-term Series Forecasting via Ultra-long Lookback Windows & \url{https://ojs.aaai.org/index.php/AAAI/article/view/35386} & \url{https://github.com/SuxinTong/IRPA} & 2025.04 \\ 
    CALF: Aligning LLMs for Time Series Forecasting via Cross-modal Fine-Tuning & \url{https://arxiv.org/pdf/2403.07300} & \url{https://github.com/Hank0626/CALF} & 2025.04 \\ 
\end{xltabular}
}

\subsubsection{Short-term Forecast}
For short-term forecasting, we focus on models capable of capturing immediate temporal patterns. The benchmark includes a selection of standard baselines and emerging architectures designed for shorter prediction horizons, as detailed in Table~\ref{tab:shortterm_papers}.

{
\small
\renewcommand{\arraystretch}{1.3}
\begin{xltabular}{\textwidth}{
    >{\raggedright\arraybackslash\hsize=0.8\hsize}X 
    >{\raggedright\arraybackslash\hsize=1.1\hsize}X 
    >{\raggedright\arraybackslash\hsize=1.1\hsize}X 
    c
}
    \caption{Selected Baselines for Short-term Time Series Forecasting Benchmark.} \label{tab:shortterm_papers} \\
    
    \toprule
    \textbf{Paper Title} & \textbf{arXiv} & \textbf{GitHub} & \textbf{Date} \\ 
    \midrule
    \endfirsthead
    
    \multicolumn{4}{c}{{\bfseries \tablename\ \thetable{} -- continued from previous page}} \\
    \toprule
    \textbf{Paper Title} & \textbf{arXiv} & \textbf{GitHub} & \textbf{Date} \\ 
    \midrule
    \endhead
    
    \midrule
    \multicolumn{4}{r}{{Continued on next page}} \\
    \bottomrule
    \endfoot
    
    \bottomrule
    \endlastfoot

    % --- Training Set Sources ---
    \multicolumn{4}{l}{\textit{\textbf{Training Set Sources}}} \\ \midrule
    Attention Is All You Need & \url{https://arxiv.org/pdf/1706.03762} & \multicolumn{1}{c}{--} & 2017.06 \\ 
    Reformer: The Efficient Transformer & \url{https://arxiv.org/pdf/2001.04451} & \url{https://github.com/google/trax} & 2020.01 \\ 
    Informer: Beyond Efficient Transformer & \url{https://arxiv.org/pdf/2012.07436} & \url{https://github.com/zhouhaoyi/Informer2020} & 2020.12 \\ 
    Autoformer: Decomposition Transformers & \url{https://arxiv.org/pdf/2106.13008} & \url{https://github.com/thuml/Autoformer} & 2021.06 \\ 
    Pyraformer: Low-Complexity Pyramidal Attention & \url{https://openreview.net/pdf?id=0EXmFzUn5I} & \url{https://github.com/ant-research/Pyraformer} & 2022.01 \\ 
    FEDformer: Frequency Enhanced Transformer & \url{https://arxiv.org/pdf/2201.12740} & \url{https://github.com/MAZiqing/FEDformer} & 2022.01 \\ 
    ETSformer: Exponential Smoothing Transformers & \url{https://arxiv.org/pdf/2202.01381} & \url{https://github.com/salesforce/ETSformer} & 2022.02 \\ 
    Non-stationary Transformers & \url{https://arxiv.org/pdf/2205.14415} & \url{https://github.com/thuml/Nonstationary_Transformers} & 2022.05 \\ 
    FiLM: Frequency improved Legendre Memory & \url{https://arxiv.org/pdf/2205.08897} & \url{https://github.com/tianzhou2011/FiLM/} & 2022.05 \\ 
    LightTS: Lightweight Time Series Classification & \url{https://arxiv.org/pdf/2302.12721} & \url{https://github.com/d-gcc/LightTS} & 2023.02 \\ 
    TimesNet: Temporal 2D-Variation Modeling & \url{https://arxiv.org/pdf/2210.02186} & \url{https://github.com/thuml/TimesNet} & 2022.10 \\ 
    Crossformer: Utilizing Cross-Dim Dependency & \url{https://openreview.net/pdf?id=vSVLM2j9eie} & \url{https://github.com/Thinklab-SJTU/Crossformer} & 2023.02 \\ 
    MICN: Multi-scale Local and Global Context & \url{https://openreview.net/pdf?id=zt53IDUR1U} & \url{https://github.com/wanghq21/MICN} & 2023.02 \\ 
    TSMixer: An All-MLP Architecture & \url{https://arxiv.org/pdf/2303.06053} & \url{https://github.com/google-research/tsmixer} & 2023.03 \\ 
    iTransformer: Inverted Transformers & \url{https://arxiv.org/pdf/2310.06625} & \url{https://github.com/thuml/iTransformer} & 2023.10 \\ 

    % --- Validation Set Source ---
    \midrule
    \multicolumn{4}{l}{\textit{\textbf{Validation Set Sources}}} \\ \midrule
    Mamba: Linear-Time Sequence Modeling & \url{https://arxiv.org/pdf/2312.00752} & \url{https://github.com/state-spaces/mamba} & 2023.12 \\ 
    TimeMixer: Decomposable Multiscale Mixing & \url{https://arxiv.org/pdf/2405.14616} & \url{https://github.com/kwuking/TimeMixer} & 2024.05 \\ 

    % --- Test Set Source ---
    \midrule
    \multicolumn{4}{l}{\textit{\textbf{Test Set Sources}}} \\ \midrule
    TimeMixer++: General Time Series Pattern & \url{https://openreview.net/pdf?id=1CLzLXSFNn} & \url{https://github.com/kwuking/TimeMixer} & 2025.05 \\ 
    Chronos-2: From Univariate to Universal Forecasting & \url{https://arxiv.org/pdf/2510.15821} & \url{https://github.com/amazon-science/chronos-forecasting} & 2025.10 \\ 
    SEMPO: Lightweight Foundation Models for Time Series Forecasting & \url{https://arxiv.org/pdf/2510.19710} & \url{https://github.com/mala-lab/SEMPO} & 2025.10 \\ 
    Amortized Control of Continuous State Space Feynman-Kac Model for Irregular Time Series & \url{https://arxiv.org/pdf/2410.05602} & \url{https://github.com/bw-park/ACSSM} & 2025.02 \\ 
    Selective Learning for Deep Time Series Forecasting & \url{https://arxiv.org/pdf/2510.25207} & \url{https://github.com/GestaltCogTeam/selective-learning} & 2025.10 \\ 
\end{xltabular}
}
\subsection{Graph Learning}
\label{sec:graph_learning}

Graph learning addresses the challenge of modeling complex relationships within data. We categorize our benchmark into two distinct sub-tasks: handling noisy labels in static graphs (NoisyGL) and modeling dynamic interactions in temporal graphs (TGL).

\subsubsection{NoisyGL}
The NoisyGL benchmark focuses on Label Noise Graph Learning. It evaluates the ability of Graph Neural Networks (GNNs) to maintain performance when node labels are unreliable or corrupted. Table~\ref{tab:noisygl_papers} summarizes the selected methods, ranging from loss correction strategies to advanced noise-tolerant architectures.

{
\small
\renewcommand{\arraystretch}{1.3}
\begin{xltabular}{\textwidth}{
    >{\raggedright\arraybackslash\hsize=1.3\hsize}X 
    >{\raggedright\arraybackslash\hsize=0.7\hsize}X 
    c
}
    \caption{Summary of Baselines and SOTA Models for Graph Learning with Noisy Labels.} \label{tab:noisygl_papers} \\
    
    \toprule
    \textbf{Paper Title} & \textbf{Source} & \textbf{Date} \\ 
    \midrule
    \endfirsthead
    
    \multicolumn{3}{c}{{\bfseries \tablename\ \thetable{} -- continued from previous page}} \\
    \toprule
    \textbf{Paper Title} & \textbf{Source} & \textbf{Date} \\ 
    \midrule
    \endhead
    
    \midrule
    \multicolumn{3}{r}{{Continued on next page}} \\
    \bottomrule
    \endfoot
    
    \bottomrule
    \endlastfoot

    % --- Training Set Sources ---
    \multicolumn{3}{l}{\textit{\textbf{Training Set Sources}}} \\ \midrule
    Semi-supervised Classification with Graph Convolutional Networks & \url{https://arxiv.org/pdf/1609.02907} & 2016.09 \\ 
    How Powerful are Graph Neural Networks? & \url{https://arxiv.org/pdf/1810.00826} & 2018.10 \\ 
    Training Deep Neural-Networks Using a Noise Adaptation Layer & \url{https://openreview.net/pdf?id=H12GRgcxg} & 2017.02 \\ 
    Making Deep Neural Networks Robust to Label Noise & \url{https://arxiv.org/pdf/1609.03683} & 2016.09 \\ 
    Co-teaching: Robust Training with Extremely Noisy Labels & \url{https://arxiv.org/pdf/1804.06872} & 2018.04 \\ 
    Symmetric Cross Entropy for Robust Learning with Noisy Labels & \url{https://arxiv.org/pdf/1908.06112} & 2019.08 \\ 
    Combating Noisy Labels by Agreement: A Joint Training Method & \url{https://arxiv.org/pdf/2003.02752} & 2020.05 \\ 
    Normalized Loss Functions for Deep Learning with Noisy Labels & \url{https://proceedings.mlr.press/v119/ma20c} & 2020.06 \\ 
    Adversarial Label-Flipping Attack and Defense for GNNs & \url{https://ieeexplore.ieee.org/document/9338299} & 2020.11 \\ 
    NRGNN: Learning a Label Noise Resistant GNN & \url{https://dl.acm.org/doi/10.1145/3447548} & 2021.08 \\ 
    Unified Robust Training for Graph Neural Networks Against Label Noise & \url{https://link.springer.com/chapter/10.1007/978} & 2021.05 \\ 
    Robust Training of Graph Neural Networks via Noise Governance & \url{https://dl.acm.org/doi/abs/10.1145/3539597} & 2023.02 \\ 
    CLNode: Curriculum Learning for Node Classification & \url{https://dl.acm.org/doi/10.1145/3539597} & 2023.02 \\ 
    Learning on Graphs under Label Noise & \url{https://ieeexplore.ieee.org/abstract/document/10096088/} & 2023.06 \\ 
    Noise-robust Graph Learning by Estimating Pairwise Interactions & \url{https://openreview.net/forum?id=r7imkFEAQb} & 2023.10 \\ 
    Robust Node Classification on Graph Data with Graph and Label Noise & \url{https://ojs.aaai.org/index.php/AAAI/article/view/29668} & 2024.03 \\ 
    Contrastive Learning of Graphs under Label Noise & \url{https://www.sciencedirect.com/science/article/pii/S0893} & 2024.04 \\ 

    % --- Validation Set Source ---
    \midrule
    \multicolumn{3}{l}{\textit{\textbf{Validation Set Sources}}} \\ \midrule
    Resurrecting Label Propagation for Graphs with Heterophily & \url{https://dl.acm.org/doi/abs/10.1145/3637528} & 2024.08 \\ 
    Mitigating Label Noise on Graph via Topological Sample Selection & \url{https://proceedings.mlr.press/v235/wu24ae.html} & 2024 \\ 

    % --- Test Set Source ---
    \midrule
    \multicolumn{3}{l}{\textit{\textbf{Test Set Sources}}} \\ \midrule
    Learning from Graph: Mitigating Label Noise via Feature Reconstruction & \url{https://dl.acm.org/doi/epdf/10.1145/3746252} & 2025.11 \\ 
    Training Robust Graph Neural Networks by Modeling Noise Dependencies & \url{https://arxiv.org/pdf/2502.19670} & 2025.02 \\ 
    JOINT GRAPH REWIRING AND FEATURE DENOISING VIA SPECTRAL RESONANCE & \url{https://openreview.net/pdf?id=zBbZ2vdLzH} & 2025.04 \\ 
    DiffGraph: Heterogeneous Graph Diffusion Model & \url{https://arxiv.org/pdf/2501.02313} & 2025.01 \\ 
\end{xltabular}
}

\subsubsection{TGL}
The TGL benchmark is dedicated to Temporal Graph Learning. This task involves modeling the continuous evolution of graph topology and node interactions over time. Table~\ref{tab:tgl_papers} outlines the baselines and recent state-of-the-art methods selected for temporal link prediction and forecasting.

{
\small
\renewcommand{\arraystretch}{1.3}
\begin{xltabular}{\textwidth}{
    >{\raggedright\arraybackslash\hsize=1.3\hsize}X 
    >{\raggedright\arraybackslash\hsize=0.7\hsize}X 
    c
}
    \caption{Selected Models and Sources for Temporal Graph Learning (TGL).} \label{tab:tgl_papers} \\
    
    \toprule
    \textbf{Paper Title} & \textbf{Source} & \textbf{Date} \\ 
    \midrule
    \endfirsthead
    
    \multicolumn{3}{c}{{\bfseries \tablename\ \thetable{} -- continued from previous page}} \\
    \toprule
    \textbf{Paper Title} & \textbf{Source} & \textbf{Date} \\ 
    \midrule
    \endhead
    
    \midrule
    \multicolumn{3}{r}{{Continued on next page}} \\
    \bottomrule
    \endfoot
    
    \bottomrule
    \endlastfoot

    % --- Training Set Sources ---
    \multicolumn{3}{l}{\textit{\textbf{Training Set Sources}}} \\ \midrule
    Temporal Graph Networks for Deep Learning on Dynamic Graphs & \url{https://arxiv.org/pdf/2006.10637} & 2020.06 \\ 
    Learning Representation over Dynamic Graphs & \url{https://openreview.net/pdf?id=HyePrhR5KX} & 2018.12 \\ 
    TLogic: Temporal Logical Rules for Explainable Link Forecasting & \url{https://arxiv.org/pdf/2112.08025} & 2021.12 \\ 
    TimeTraveler: Reinforcement Learning for Temporal KG Forecasting & \url{https://arxiv.org/pdf/2109.04101} & 2021.09 \\ 
    Temporal Knowledge Graph Reasoning Based on Evolutional RL & \url{https://arxiv.org/pdf/2104.10353} & 2021.04 \\ 
    Complex Evolutional Pattern Learning for Temporal KG Reasoning & \url{https://arxiv.org/pdf/2203.07782} & 2022.03 \\ 
    Spatio-Temporal Heterogeneous Graph Neural Networks & \url{https://www.mdpi.com/2079-9292/12/6/1293} & 2023.12 \\ 
    Temporal Graph Benchmark for Machine Learning on Temporal Graphs & \url{https://arxiv.org/pdf/2307.01026} & 2023.07 \\ 

    % --- Validation Set Source ---
    \midrule
    \multicolumn{3}{l}{\textit{\textbf{Validation Set Sources}}} \\ \midrule
    Efficient Neural Common Neighbor for Temporal Graph Link Prediction & \url{https://arxiv.org/pdf/2406.07926} & 2024.06 \\ 
    History repeats itself: A Baseline for Temporal KG Forecasting & \url{https://dl.acm.org/doi/10.24963/ijcai.2024/444} & 2024.08 \\ 

    % --- Test Set Source ---
    \midrule
    \multicolumn{3}{l}{\textit{\textbf{Test Set Sources}}} \\ \midrule
    Revisiting Node Affinity Prediction in Temporal Graphs & \url{https://arxiv.org/pdf/2510.06940} & 2025.10 \\
    DyG-Mamba: Continuous State Space Modeling on Dynamic Graphs & \url{https://arxiv.org/pdf/2408.06966} & 2025.12 \\ 
\end{xltabular}
}

\section{Baselines}
\label{appendix:baselines}

In this section, we provide detailed descriptions of the baseline methods used in our experiments. As stated in the main text, we compare against two categories of representative baseline methods: (i) general-purpose coding agents and (ii) workflow-based agents specifically designed for paper reproduction. All baselines are evaluated under identical experimental conditions as described in Section~\ref{subsec:exp_setup}.

\subsection{General-Purpose Coding Agents}

\subsubsection{ReAct}
\label{appendix:baseline:react}

ReAct is a foundational reasoning-and-acting framework that synergizes reasoning traces and task-specific actions in an interleaved manner. The core paradigm of ReAct augments the agent's action space to include both external actions and internal reasoning traces (thoughts). Specifically, given an observation $o_t$ at time step $t$, the agent generates either an action $a_t$ to interact with the environment or a thought $\hat{a}_t$ to reason over the current context $c_t = (o_1, a_1, \dots, o_{t-1}, a_{t-1}, o_t)$.

The key innovation of ReAct lies in its ability to:
\begin{itemize}
    \item \textbf{Reason to Act:} Generate reasoning traces to decompose tasks, track progress, handle exceptions, and maintain working memory to support action generation.
    \item \textbf{Act to Reason:} Interact with external environments (e.g., Wikipedia API, knowledge bases) to gather additional information that grounds the reasoning process and reduces hallucination.
\end{itemize}

In the context of paper reproduction, ReAct operates by prompting large language models with few-shot in-context examples. Each example consists of a human trajectory of interleaved thoughts, actions, and environment observations. The agent alternates between generating reasoning traces (e.g., ``I need to implement the model architecture described in Section 3'') and executing actions (e.g., writing code, running commands). This paradigm has been widely adopted in agentic systems due to its interpretability and ability to handle complex multi-step tasks.

For our experiments, we implement ReAct following its original prompting setup, where the agent receives the target paper as input and generates code through iterative reasoning and acting cycles. The agent has access to file system operations, code execution capabilities, and web search functionality to retrieve relevant information during the reproduction process.

\subsubsection{OpenHands}
\label{appendix:baseline:openhands}

OpenHands (formerly known as OpenDevin) is a state-of-the-art open-source coding agent platform designed for generalist AI software developers. It has demonstrated strong performance on software engineering benchmarks such as SWE-Bench and various web-based tasks. OpenHands provides a comprehensive framework that enables agents to interact with the world in ways similar to human developers.

The architecture of OpenHands consists of three main components:

\paragraph{\textbf{Agent Abstraction and Event Stream.}} OpenHands defines agents through a simple yet powerful abstraction centered around a \texttt{step} function that takes the current state (including an event stream of past actions and observations) and generates appropriate actions. The event stream architecture provides a chronological collection of all interactions, enabling the agent to maintain context throughout long-horizon tasks.

\paragraph{\textbf{Runtime Environment.}} OpenHands provides a docker-sandboxed operating system equipped with:
\begin{itemize}
    \item A bash shell for executing arbitrary command-line operations
    \item An IPython server for interactive Python code execution
    \item A Chromium-based web browser with BrowserGym integration for web navigation
\end{itemize}

This runtime environment allows agents to perform tasks comparable to human software engineers, including creating and editing complex software, executing code in isolated sandboxes, and browsing websites to gather information.

\paragraph{\textbf{Action Space.}} Inspired by the CodeAct paradigm, OpenHands connects agents with the environment through a core set of general actions:
\begin{itemize}
    \item \texttt{IPythonRunCellAction}: Execute arbitrary Python code
    \item \texttt{CmdRunAction}: Execute bash commands
    \item \texttt{BrowserInteractiveAction}: Interact with web browsers using a domain-specific language
    \item \texttt{FileEditAction}: Create and modify files in the workspace
\end{itemize}

For paper reproduction tasks, OpenHands agents receive the target paper PDF and are tasked with generating a complete, executable implementation. The agent iteratively reads the paper, plans the implementation, writes code files, executes tests, and debugs errors until a working implementation is achieved. OpenHands' powerful action primitives and sandboxed execution environment make it particularly suitable for complex software development tasks.

In our experiments, we configure OpenHands with its default CodeAct-based agent, which has shown strong generalization capabilities across diverse coding tasks. All experiments use identical computational resources and time limits as specified in the fair comparison protocol (Section~\ref{subsec:exp_setup}).

\subsection{Workflow-Based Agents for Paper Reproduction}

\subsubsection{Paper2Code}
\label{appendix:baseline:paper2code}

Paper2Code is a specialized multi-agent LLM framework designed specifically to transform machine learning papers into executable code repositories. Unlike general-purpose coding agents, Paper2Code explicitly models the typical software development lifecycle and decomposes the paper-to-code transformation into three structured stages.

\paragraph{\textbf{Planning Stage.}} The planning stage transforms unstructured paper content into implementation-level abstractions through four sequential subcomponents:

\begin{enumerate}
    \item \textbf{Overall Plan ($o$):} Extracts a high-level summary of core components and functionalities, including model components, training objectives, data processing steps, and evaluation protocols. Formally, $M_{\text{plan}}^{(1)}(R) := \text{LLM}(\mathcal{T}_{\text{plan}}^{(1)}(R)) \to o$, where $R$ is the input paper.
    
    \item \textbf{Architecture Design ($d$):} Defines the repository-level architecture by generating a file list, class diagrams (static representations of classes and attributes), and sequence diagrams (dynamic interactions between components). This is formalized as $M_{\text{plan}}^{(2)}(R, o) := \text{LLM}(\mathcal{T}_{\text{plan}}^{(2)}(R, o)) \to d$.
    
    \item \textbf{Logic Design ($l$):} Specifies the execution flow by producing an ordered file list that dictates the sequence in which files should be implemented, considering inter-file dependencies. This ensures that dependent files are generated after their dependencies: $M_{\text{plan}}^{(3)}(R, o, d) := \text{LLM}(\mathcal{T}_{\text{plan}}^{(3)}(R, o, d)) \to l$.
    
    \item \textbf{Configuration Generation ($g$):} Synthesizes a configuration file (e.g., \texttt{config.yaml}) containing hyperparameters, model settings, and runtime options: $M_{\text{plan}}^{(4)}(R, o, d, l) := \text{LLM}(\mathcal{T}_{\text{plan}}^{(4)}(R, o, d, l)) \to g$.
\end{enumerate}

The complete planning output is $P = \{o, d, l, g\}$.

\paragraph{\textbf{Analysis Stage.}} Following planning, the analysis phase interprets implementation-level details for each file. For each file $f_i$ identified during planning, the analysis agent generates a detailed specification $a_i$ describing functional goals, input-output behaviors, intra- and inter-file dependencies, and algorithmic specifications:
\begin{equation}
M_{\text{analysis}}(R, P, f_i) := \text{LLM}(\mathcal{T}_{\text{analysis}}(R, P, f_i)) \to a_i
\end{equation}

\paragraph{\textbf{Generation Stage.}} The final coding phase generates the complete repository by producing each file sequentially according to the execution order determined in the logic design. For each file $f_i$, the code $c_i$ is generated with full awareness of previously generated files:
\begin{equation}
M_{\text{code}}(R, P, f_i, a_i, \{c_1, \dots, c_{i-1}\}) := \text{LLM}(\mathcal{T}_{\text{code}}(R, P, f_i, a_i, \{c_1, \dots, c_{i-1}\})) \to c_i
\end{equation}

The complete repository is $C = \{c_i\}_{i=1}^{n=|F|}$.

Paper2Code's structured approach explicitly addresses the challenges of maintaining cross-file consistency and handling the complexity of scientific papers. By decomposing the task into planning, analysis, and generation stages, it reduces the cognitive load on the LLM at each step and ensures that implementation decisions are grounded in a coherent architectural plan.

In our experiments, we implement Paper2Code following its official methodology, using the same backbone LLM (Claude Opus 4.5) and evaluation protocol as all other baselines to ensure fair comparison.

\subsubsection{DeepCode}
\label{appendix:baseline:deepcode}

DeepCode is a fully autonomous framework designed for high-fidelity document-to-repository synthesis, with a particular focus on scientific paper reproduction. DeepCode addresses the fundamental conflict between information overload (lengthy, multimodal papers) and context bottlenecks (finite LLM context windows) through principled information-flow management. The framework treats repository synthesis as a channel optimization problem and orchestrates four strategic information operations.

\paragraph{\textbf{Blueprint Generation (Source Compression).}} The first phase distills unstructured, lengthy paper content into a structured, machine-readable implementation blueprint through hierarchical content segmentation and multi-agent specification analysis:

\begin{enumerate}
    \item \textbf{Hierarchical Content Segmentation:} The paper $\mathcal{D}$ is parsed into a structured representation $S = \{s_1, s_2, \dots, s_K\}$ based on section and subsection titles. Each chunk $s_k$ is stored as a key-value pair $(h_k, c_k)$, where $h_k$ is the heading and $c_k$ is the content.
    
    \item \textbf{Multi-Agent Specification Analysis:} Two specialized agents operate in parallel:
    \begin{itemize}
        \item \textit{Concept Agent:} Builds a holistic understanding by mapping the paper's conceptual structure, identifying core contributions, and outlining necessary components for reproduction.
        \item \textit{Algorithm Agent:} Extracts low-level technical details including algorithms, mathematical formulations, model architectures, training procedures, and hyperparameters.
    \end{itemize}
    
    \item \textbf{Blueprint Synthesis:} A Code Planning Agent synthesizes the outputs from both agents into a unified Implementation Blueprint $\mathcal{B}$ containing:
    \begin{itemize}
        \item Project file hierarchy with prioritized implementation order
        \item Component specifications mapping modules to algorithmic pseudocode
        \item Verification protocol defining experimental setup and success criteria
        \item Execution environment specifying dependencies and hardware requirements
        \item Staged development plan with phased implementation roadmap
    \end{itemize}
\end{enumerate}

\paragraph{\textbf{Code Generation (Structured Indexing and Knowledge Injection).}} The second phase synthesizes the code repository while preventing context saturation through two integrated mechanisms:

\begin{enumerate}
    \item \textbf{Stateful Code Memory (CodeMem):} Maintains a compressed, structured representation of the repository state. For each file $c_t$ generated at step $t$, the system:
    \begin{itemize}
        \item Constructs generation context: $\mathcal{X}_t = (\mathcal{B}, \text{SelectRelevantMemory}(\mathcal{M}_{t-1}, \hat{c}_t))$
        \item Generates code: $c_t = \mathcal{L}(\mathcal{X}_t)$
        \item Updates memory: $\mathcal{M}_t = \mathcal{M}_{t-1} \cup \{m_t\}$, where $m_t$ contains:
        \begin{itemize}
            \item Core purpose $\mathcal{P}_t$: Natural language summary of file's role
            \item Public interface $\mathcal{I}_t$: Externally accessible classes and functions
            \item Dependency edges $\mathcal{E}_t$: Afferent and efferent couplings
        \end{itemize}
    \end{itemize}
    
    \item \textbf{CodeRAG (Conditional Knowledge Injection):} Grounds synthesis in a pre-indexed corpus of relevant repositories:
    \begin{itemize}
        \item \textit{Repository Indexing:} Analyzes relevant repositories $\mathcal{R} = \{R_1, R_2, \dots, R_K\}$ to build a structured index $\mathcal{I}$ containing relationship tuples $(c_s', \hat{c}_t, \tau, \sigma, \gamma)$, where $c_s'$ is a source file, $\hat{c}_t$ is a target file, $\tau$ is the relationship type, $\sigma$ is confidence, and $\gamma$ is actionable context.
        \item \textit{Adaptive Retrieval:} At each generation step, the agent decides whether to retrieve external knowledge: $r_t = \delta(\mathcal{X}_t, \hat{c}_t)$. If $r_t = 1$, augmented context is created: $\mathcal{X}_t' = \mathcal{X}_t \cup \{\text{Retrieve}(\mathcal{I}, \hat{c}_t)\}$.
    \end{itemize}
\end{enumerate}

\paragraph{\textbf{Automated Verification (Closed-Loop Error Correction).}} The final phase ensures functional correctness through two sequential stages:

\begin{enumerate}
    \item \textbf{Static Analysis:} An Analysis Agent $\mathcal{A}_{\text{static}}$ inspects the repository $\mathcal{P}$ against the blueprint $\mathcal{B}$ to identify structural discrepancies and code quality deficiencies: $\mathcal{R}_{\text{static}} = \mathcal{A}_{\text{static}}(\mathcal{P}, \mathcal{B})$. A Modification Agent then applies targeted fixes: $\mathcal{P}' = \mathcal{A}_{\text{modify}}(\mathcal{P}, \mathcal{R}_{\text{static}})$.
    
    \item \textbf{Sandbox Execution:} A Sandbox Agent executes the repository in an isolated environment and iteratively corrects errors:
    \begin{equation}
    \mathcal{T}_j = \mathcal{E}_{\text{sandbox}}(\mathcal{P}_j'), \quad \mathcal{P}_{j+1}' = \Phi_{\text{LSP}}(\mathcal{P}_j', \mathcal{T}_j^{\text{error}})
    \end{equation}
    This loop continues until execution succeeds or maximum iterations are reached.
\end{enumerate}

DeepCode's principled information-flow management enables it to maximize the signal-to-noise ratio within finite context windows, effectively addressing the challenges of long-range specification preservation, cross-file consistency, and implicit knowledge gaps. The framework has demonstrated state-of-the-art performance on the PaperBench benchmark, surpassing both commercial agents and human experts on key reproduction metrics.

In our experiments, DeepCode is configured with Claude Opus 4.5 as the backbone LLM, consistent with all other baselines. We adopt its default hyperparameters and follow the official implementation to ensure faithful reproduction of the reported results.

\subsection{Implementation Details}

All baseline methods are evaluated under the fair comparison protocol described in Section~\ref{subsec:exp_setup}:

\begin{itemize}
    \item \textbf{Unified Backbone:} All methods use Claude Opus 4.5 as the backbone LLM with identical API configurations (temperature, max tokens, system prompts).
    \item \textbf{Computational Resources:} Experiments are conducted on a homogeneous cluster of 64 NVIDIA A800 GPUs (80GB memory) with consistent software environments (Python 3.10, PyTorch 2.7, CUDA 12.6).
    \item \textbf{Computational Budget:} Each method is allocated 5 reproduction attempts per paper with a maximum of 50 LLM interaction rounds.
    \item \textbf{Prompt Engineering:} For each baseline, we adopt officially recommended prompt templates and hyperparameters when available; otherwise, we conduct systematic prompt tuning on a held-out validation dataset.
    \item \textbf{Evaluation Pipeline:} All methods receive identical inputs (target paper PDF, fixed evaluation datasets, pre-configured environment dependencies) and are evaluated using the same execution harness with deterministic random seeds and timeout constraints (2 hours per reproduction attempt).
    \item \textbf{Statistical Rigor:} Results are averaged over 5 independent runs with different random seeds, reporting both mean performance and standard deviation.
\end{itemize}

This rigorous protocol ensures that performance differences reflect genuine methodological advantages rather than implementation artifacts or unfair resource allocation.

\section{Prompts}
\label{app:prompts}

\noindent
This appendix provides the prompt templates used at different stages of the pipeline.

\subsection{Full-Pipeline Orchestration}
\label{app:prompts:orchestration}

\paragraph{Reproduction Agent Orchestration.}
This prompt serves as the top-level controller of the full pipeline. It integrates local reusable/adaptable modules, execution feedback, and retrieved global knowledge, and outputs either an initial build plan or an iterative repair plan.
\begin{promptbox}{Reproduction Agent Orchestration}
\begin{lstlisting}[style=prompttext]
---SYSTEM_PROMPT---
You are an expert scientific reproduction agent. Your objective is to produce executable code that faithfully reproduces the target paper under the provided environment constraints.

You must integrate:
- Local implementation units (reuse/adapt/new) from neighbor papers.
- Execution feedback from previous attempts.
- Retrieved collective knowledge from subgraph-level induction.

Core principles:
- Prioritize implementation fidelity over stylistic refactoring.
- Make minimal, high-confidence changes per iteration.
- Explicitly report assumptions when paper details are missing.
- Keep outputs machine-readable.

---USER_PROMPT---
# Task
Generate either:
1. A full implementation draft when no code exists yet; or
2. A targeted repair plan and patch when current code and feedback are provided.

# Inputs
Target Paper ID: {{ target_paper_id }}
Target Paper Method/Experiments:
{{ target_paper_text }}

Optional Initial Code:
{{ current_code }}

Optional Local Modules (from Node-Level Relation-Aware Aggregation outputs):
{{ selected_modules }}

Optional Execution Feedback:
{{ execution_feedback }}

Optional Collective Knowledge Context:
{{ injected_knowledge_context }}

Environment/Constraints:
{{ runtime_constraints }}

# Required Procedure
1. Identify core implementation units required by the target paper.
2. Reuse selected local modules when applicable.
3. Apply collective knowledge only when trigger conditions match.
4. If feedback exists, diagnose root cause before proposing edits.
5. Prefer minimal edits that are testable in the next run.

# Output Format (JSON only)
{
  "mode": "initial_build|iterative_repair",
  "assumptions": ["<assumption or unknown>"],
  "plan": [
    "<ordered implementation or repair step>"
  ],
  "files_to_create_or_modify": [
    "<relative file path>"
  ],
  "code_actions": [
    {
      "file": "<path>",
      "action": "create|modify",
      "content_or_diff": "<code snippet or minimal diff>"
    }
  ],
  "verification": [
    "<what to run/check next>"
  ]
}

# Requirements
- Output JSON only.
- Use "unknown" for missing critical details.
- Do not output narrative outside JSON.

\end{lstlisting}
\end{promptbox}

\subsection{Semantic Scientific Graph Pruning}
\label{app:prompts:ssgp}

\paragraph{Paper Summary for Relevance Assessment.}
This prompt extracts implementation-oriented summaries from the Method and Experiments sections of each paper. The structured summary is used as the input representation for implementation relevance ranking.
\begin{promptbox}{Paper Summary for Relevance Assessment}
\begin{lstlisting}[style=prompttext]
---SYSTEM_PROMPT---
You are an expert in scientific paper reproduction and implementation analysis. Your goal is to extract implementation-critical details from a paper in a compact, structured form.

Key principles:
- Focus strictly on implementation-relevant content.
- Ignore background, theory, and related work.
- If information is missing or ambiguous, return "unknown" rather than guessing.
- Be concise but precise.

---USER_PROMPT---
# Task
Produce a structured, implementation-oriented summary for downstream relevance ranking and reuse analysis.

# Input
Paper ID: {{ paper_id }}
Method and Experiments Text:
{{ method_experiments }}

# Output Format (JSON only)
{
  "paper_id": "<id>",
  "method_summary": "<3-6 sentences describing implementation mechanics>",
  "components": [
    {"name": "<module>", "role": "<what it does>", "notes": "<key implementation detail or unknown>"}
  ],
  "architecture": {
    "backbone": "<model family or unknown>",
    "key_blocks": ["<block names>"],
    "input_outputs": "<input/output shapes or unknown>"
  },
  "training": {
    "optimizer": "<name or unknown>",
    "learning_rate": "<value or unknown>",
    "schedule": "<value or unknown>",
    "batch_size": "<value or unknown>",
    "epochs": "<value or unknown>",
    "losses": ["<loss names>"],
    "regularization": "<weight decay/dropout/etc or unknown>"
  },
  "hyperparameters": {
    "key": "value or unknown"
  },
  "data": {
    "datasets": ["<dataset names>"],
    "preprocessing": "<critical preprocessing or unknown>",
    "splits": "<train/val/test or unknown>"
  },
  "evaluation": {
    "metrics": ["<metric names>"],
    "protocol": "<evaluation setup or unknown>"
  },
  "implicit_decisions": [
    "<implementation choices implied but not explicit>"
  ]
}

# Requirements
- Implementation-focused, avoid theory and motivation.
- Use concise, technical language.
- If a detail is missing, output "unknown".
- Output JSON only.
\end{lstlisting}
\end{promptbox}

\paragraph{Implementation Relevance Ranking.}
This prompt is executed by $K$ independent LLM reviewers to perform list-wise ranking over all candidate neighbors by implementation-level relevance. The aggregated mean rank and rank variance are then used to compute neighborhood selection scores in semantic graph pruning.
\begin{promptbox}{Implementation Relevance Ranking}
\begin{lstlisting}[style=prompttext]
---SYSTEM_PROMPT---
You are an expert in implementation-level paper analysis. Rank candidate neighbor papers by whether their code would help reproduce the target paper.

Focus on concrete implementation overlap: architecture, training pipeline, preprocessing, and evaluation. Do not use citation relationships or high-level conceptual similarity.

---USER_PROMPT---
# Task
Given the target summary and a set of candidate summaries, rank ALL candidates by implementation relevance to the target paper (higher relevance = smaller rank).

# Input
Target Summary:
{{ target_summary }}

Candidate Summaries (list):
{{ candidate_summaries }}

# Criteria (most important first)
1. Shared architecture and components (backbone, blocks, encoders)
2. Similar training pipeline and evaluation protocol
3. Reuse potential of modules or code structure
4. Overlap in loss functions, preprocessing, or hyperparameter strategy

# Output Format (JSON only)
{
  "ranking": [
    {
      "paper_id": "<candidate id>",
      "rank": <integer, 1 = most relevant>,
      "confidence": <float in [0,1]>,
      "rationale": "<one sentence focused on implementation overlap>",
      "evidence": [
        "<short bullet: overlap or mismatch>"
      ]
    }
  ],
  "unknown": [
    {
      "paper_id": "<candidate id>",
      "missing_fields": ["architecture|training|data|evaluation|implicit_decisions"],
      "note": "<what is missing and how it affects ranking>"
    }
  ]
}

# Requirements
- You MUST rank all candidates; ranks must form a permutation of 1..N (no gaps).
- Implementation-level relevance only; no citation-based arguments.
- If evidence is insufficient, keep the candidate in the ranking but lower confidence and record missing info in "unknown".
- Be concise: rationale must be one sentence.
- Output JSON only.
\end{lstlisting}
\end{promptbox}

\subsection{Node-Level Relation-Aware Aggregation}
\label{app:prompts:nodelevel}

\paragraph{Implementation Unit Relation Analysis.}
This prompt identifies reusable units, adaptable units, and newly required units between the target paper and each neighbor paper. It produces structured relation annotations used in downstream API encapsulation.
\begin{promptbox}{Implementation Unit Relation Analysis}
\begin{lstlisting}[style=prompttext]
---SYSTEM_PROMPT---
You are an expert in code reuse and adaptation for scientific reproduction. Compare a target paper to a neighbor paper with available code and identify reusable, adaptable, and new implementation units.

Use module-level granularity (e.g., encoder, loss, sampler, data loader, training loop, evaluation). Provide concrete, actionable adaptation instructions.

---USER_PROMPT---
# Task
Identify implementation-unit relations between the target paper and the neighbor paper.

# Input
Target Paper Text:
{{ target_paper_text }}

Neighbor Paper Text:
{{ neighbor_paper_text }}

Neighbor Code Index (files/functions/classes):
{{ neighbor_code_index }}

# Output Format (JSON only)
{
  "reusable_units": [
    {
      "unit_name": "<module>",
      "description": "<why reusable>",
      "code_location": "<file:line or symbol>",
      "evidence": "<text evidence or unknown>",
      "risk": "<low|medium|high>"
    }
  ],
  "adaptable_units": [
    {
      "unit_name": "<module>",
      "description": "<difference>",
      "code_location": "<file:line or symbol>",
      "diff_instruction": "<concrete modification>",
      "evidence": "<text evidence or unknown>",
      "risk": "<low|medium|high>"
    }
  ],
  "new_units": [
    {
      "unit_name": "<module>",
      "description": "<what to implement>",
      "reason": "<no counterpart in neighbor>",
      "evidence": "<text evidence or unknown>"
    }
  ]
}

# Requirements
- Use module-level granularity.
- diff_instruction must be actionable and specific.
- Prefer short, precise evidence.
- Output JSON only.
\end{lstlisting}
\end{promptbox}

\paragraph{API Encapsulation for Reuse/Adapt/New Units.}
This prompt converts analyzed implementation units into callable APIs. It preserves reusable interfaces, applies explicit modifications to adaptable units, and generates placeholders for new units.
\begin{promptbox}{API Encapsulation}
\begin{lstlisting}[style=prompttext]
---SYSTEM_PROMPT---
You are an expert developer. Package implementation units into callable, self-contained APIs suitable for reuse or adaptation.

Follow existing interfaces where possible and keep changes minimal. For new units, provide a stub with a clear TODO.

---USER_PROMPT---
# Task
Encapsulate implementation units into callable APIs.

# Input
Relation Annotation:
{{ relation_annotation }}

Extracted Code Snippets:
{{ code_snippets }}

# Output Format (JSON only)
[
  {
    "api_name": "<name>",
    "kind": "reuse|adapt|new",
    "source": "<paper_id or file path>",
    "signature": "<function signature>",
    "dependencies": ["<import or helper>"],
    "code": "<complete python code for the API>",
    "notes": "<brief notes or unknown>"
  }
]

# Requirements
- Preserve original interfaces for reuse/adapt unless required by diff_instruction.
- For adapt, apply the diff_instruction directly and mention it in notes.
- For new, output a stub that raises NotImplementedError with a TODO.
- Output JSON only.
\end{lstlisting}
\end{promptbox}

\paragraph{Neighborhood Aggregation and Module Selection.}
This prompt selects the best candidate API for each implementation unit according to edge weight and reuse bias. The result is the initial implementation scaffold aggregated from the pruned neighborhood.
\begin{promptbox}{Neighborhood Aggregation and Selection}
\begin{lstlisting}[style=prompttext]
---SYSTEM_PROMPT---
You are an expert in selecting implementation modules based on relevance weights and reuse preference. Assemble the best initial implementation from candidate APIs.

---USER_PROMPT---
# Task
Select the best API for each implementation unit.

# Input
Candidate APIs per unit:
{{ candidate_apis }}

Edge Weights:
{{ edge_weights }}

Reuse Bias (beta): {{ beta }}

# Scoring Rule
p(api_i) = w(vt, vi) + beta * I[api_i is reuse]

# Output Format (JSON only)
{
  "selected": [
    {
      "unit_name": "<module>",
      "chosen_api": "<api_name>",
      "score": "<numeric>",
      "reason": "<short reason>",
      "alternatives": ["<api_name>"]
    }
  ],
  "deferred": [
    {
      "unit_name": "<module>",
      "reason": "<no suitable api>",
      "next_step": "<suggested action>"
    }
  ]
}

# Requirements
- Prefer reuse when scores are close.
- Provide brief, technical reasons.
- Output JSON only.
\end{lstlisting}
\end{promptbox}

\subsection{Execution-Feedback Refinement}
\label{app:prompts:exec}

\paragraph{Execution-Feedback Diagnosis and Repair Plan.}
This prompt takes current code, runtime feedback, and metric gaps as input, and outputs a structured repair plan. The plan specifies edit targets, concrete code modifications, and expected verification outcomes for iterative refinement.
\begin{promptbox}{Execution-Feedback Refinement}
\begin{lstlisting}[style=prompttext]
---SYSTEM_PROMPT---
You are an expert debugging and reproduction agent. Diagnose execution feedback and propose precise, minimal code edits.

Prioritize root-cause analysis and actionable changes. If evidence is insufficient, propose the smallest safe change and state the uncertainty.

---USER_PROMPT---
# Task
Analyze execution feedback and propose a repair plan.

# Input
Target Paper Text:
{{ target_paper_text }}

Current Implementation:
{{ current_code }}

Execution Feedback:
{{ execution_feedback }}

Current Metrics:
{{ current_metrics }}

Reference Metrics:
{{ reference_metrics }}

# Output Format (JSON only)
{
  "diagnosis": "<what failed or underperformed>",
  "root_cause": "<most likely cause>",
  "edit_units": ["<module or file>"],
  "edits": [
    {
      "file": "<file path>",
      "change_type": "add|modify|delete",
      "diff": "<minimal diff or instruction>",
      "risk": "<low|medium|high>"
    }
  ],
  "expected_outcome": "<how to verify>",
  "fallback": "<alternative if fix fails or unknown>"
}

# Requirements
- Edits must be concrete and implementable.
- If evidence is insufficient, state uncertainty and propose minimal fix.
- Output JSON only.
\end{lstlisting}
\end{promptbox}

\subsection{Graph-Level Knowledge Induction and Injection}
\label{app:prompts:graphlevel}

\paragraph{Subgraph-Level Knowledge Induction.}
This prompt induces recurring, transferable implementation practices from collected $(C^{(\mathrm{ref})}_v, O_v)$ execution-feedback outcomes within each implementation-coherent subgraph. It outputs structured knowledge entries filtered by recurrence and stable reproduction gains on the validation set $V_{\mathrm{val}}$.
\begin{promptbox}{Subgraph-Level Knowledge Induction}
\begin{lstlisting}[style=prompttext]
---SYSTEM_PROMPT---
You are an expert in extracting transferable implementation knowledge from multiple execution-feedback outcomes. Identify recurring practices and pitfalls that generalize across papers in a subgraph.

Only keep patterns that meet the frequency threshold and yield stable reproduction gains on the validation set V_val. Avoid dataset-specific tricks unless explicitly generalizable.

---USER_PROMPT---
# Task
Induce recurring, transferable knowledge from a subgraph of papers.

# Input
Task: {{ task_name }}
Domain: {{ domain }}
Subgraph ID: {{ subgraph_id }}
Minimum Frequency: {{ min_frequency }}

Execution-Feedback Outcomes:
{{ execution_feedback_outcomes }}

# Output Format (JSON only)
[
  {
    "pattern": "<recurring practice or pitfall>",
    "trigger": "<observable condition>",
    "action": "<concrete step>",
    "rationale": "<why it helps>",
    "verification": "<how to check success>",
    "scope": "<applicable models/datasets>",
    "frequency": "<count/total>",
    "confidence": "low|medium|high",
    "evidence": ["<paper_id or short signal>"]
  }
]

# Requirements
- Include only patterns with frequency >= min_frequency and that are supported by stable gains on V_val (otherwise discard).
- Actions must be specific and actionable.
- Output JSON only.
\end{lstlisting}
\end{promptbox}

\paragraph{Knowledge Injection for Target Reproduction.}
This prompt transforms retrieved subgraph knowledge into concise, actionable context aligned with the current target paper. The injected context is consumed by the reproduction agent in generation or repair stages.
\begin{promptbox}{Knowledge Injection}
\begin{lstlisting}[style=prompttext]
---SYSTEM_PROMPT---
You are an expert in translating induced knowledge entries into actionable implementation guidance for a specific target paper.

Your goal is to inject only relevant knowledge, avoid noise, and output compact guidance that can be consumed by the reproduction agent.

---USER_PROMPT---
# Task
Given retrieved subgraph-level knowledge and current target-paper context, produce an actionable injected context package.

# Inputs
Target Paper ID: {{ target_paper_id }}
Target Paper Summary:
{{ target_summary }}

Current Implementation State (optional):
{{ current_state }}

Retrieved Knowledge Entries:
{{ retrieved_knowledge_entries }}

# Selection Rules
1. Keep entries whose trigger clearly matches target context.
2. Prefer high-confidence and higher-frequency entries.
3. Remove conflicting entries; keep the one with stronger evidence.
4. Exclude dataset-specific tricks unless target dataset matches.

# Output Format (JSON only)
{
  "target_paper_id": "<id>",
  "selected_entries": [
    {
      "pattern": "<knowledge pattern>",
      "trigger_match": "<why matched>",
      "action": "<concrete action>",
      "priority": "high|medium|low",
      "evidence": "<frequency/confidence summary>"
    }
  ],
  "rejected_entries": [
    {
      "pattern": "<knowledge pattern>",
      "reason": "<why not used>"
    }
  ],
  "injected_context": [
    "<short actionable instruction for downstream agent>"
  ]
}

# Requirements
- Output JSON only.
- Keep injected_context concise and executable.
- Use "unknown" if a match cannot be determined.
\end{lstlisting}
\end{promptbox}

\end{document}